\newif\ifarXiv

\arXivtrue % or \arXivfalse

\documentclass{article}

\usepackage{microtype}
\usepackage{graphicx}
\usepackage{booktabs}
\usepackage{hyperref}

\ifarXiv
\usepackage{fullpage}
\usepackage{authblk}
\else
\usepackage[accepted]{icml2023}
\fi

\usepackage{amsmath}
\usepackage{amssymb}
\usepackage{mathtools}
\usepackage{amsthm}
\usepackage{natbib}

\usepackage[utf8]{inputenc}
\usepackage{algorithm}
\usepackage[noend]{algpseudocode}
\usepackage{amsfonts}
\usepackage{bm}
\usepackage{caption}
\usepackage{subcaption}
\usepackage{bbm}
\usepackage{textcomp}
\usepackage{pgfplots}
\usepackage{bbm}
\usepackage{wrapfig}
\usepackage[]{youngtab}
\usepackage{tikz}
\usepackage{tikz-cd}
\usetikzlibrary{automata, positioning}
\usepackage{comment}
\usepackage{stmaryrd}
\usepackage{theoremref}
\usepackage{mathrsfs}

\usepackage{color-edits}
\addauthor{sw}{blue}
\addauthor{rr}{red}
\addauthor{jw}{brown}
\addauthor{ar}{red}

\usepackage[english]{babel}
\usepackage[utf8]{inputenc}
\usepackage{fancyhdr}
\pgfplotsset{width=10cm,compat=1.9}

\renewcommand{\P}{\mathbb{P}}
\newcommand{\N}{\mathbb{N}}
\newcommand{\E}{\mathbb{E}}
\newcommand{\R}{\mathbb{R}}

\newcommand{\calF}{\mathcal{F}}
\newcommand{\calX}{\mathcal{X}}
\newcommand{\calY}{\mathcal{Y}}
\newcommand{\calH}{\mathcal{H}}
\newcommand{\calG}{\mathcal{G}}
\newcommand{\calL}{\mathcal{L}}
\newcommand{\calR}{\mathcal{R}}

\newcommand{\calZ}{\mathcal{Z}}

\newcommand{\calS}{\mathcal{S}}

\renewcommand{\square}{\perp}

\newtheorem{informal}{Informal Theorem}

\newtheorem{thm}{Theorem}
\newtheorem{lem}{Lemma}
\newtheorem{condition}{Condition}
\newtheorem{definition}{Definition}

\usepackage{thm-restate}
\usepackage[font=small,labelfont=bf,aboveskip=5pt]{caption}
\usepackage{hyperref}
\hypersetup{
	colorlinks=false,
	bookmarks=true,
	breaklinks=true,
	hidelinks=true,
	pdfpagemode=empty,
}
\usepackage{nameref}
\usepackage[noabbrev, capitalize, nameinlink]{cleveref}
\crefname{prop}{Proposition}{Propositions}
\crefname{rmk}{Remark}{Remarks}
\crefname{cor}{Corollary}{Corollaries}
\crefname{claim}{Claim}{Claims}
\crefname{lemma}{Lemma}{Lemmata}
\crefname{example}{Example}{Examples}
\crefname{corollary}{Corollary}{Corollaries}

\usepackage{times}

\begin{document}

\ifarXiv
\title{Fully Adaptive Composition in Differential Privacy}
\author[1]{Justin Whitehouse}
\author[1]{Aaditya Ramdas}
\author[2]{Ryan Rogers}
\author[1]{Zhiwei Steven Wu}
\affil[1]{Carnegie Mellon University}
\affil[2]{Data and AI Foundations. LinkedIn}

\maketitle 
\else
\icmltitlerunning{Fully-adaptive Composition in Differential Privacy}
\twocolumn[
\icmltitle{Fully-adaptive Composition in Differential Privacy}

\begin{icmlauthorlist}
\icmlauthor{Justin Whitehouse}{cmu}
\icmlauthor{Aaditya Ramdas}{cmu}
\icmlauthor{Ryan Rogers}{link}
\icmlauthor{Zhiwei Steven Wu}{cmu}

\end{icmlauthorlist}

\icmlaffiliation{cmu}{Carnegie Mellon University}
\icmlaffiliation{link}{LinkedIn}

\icmlcorrespondingauthor{Justin Whitehouse}{jwhiteho@andrew.cmu.edu}

\icmlkeywords{Machine Learning, ICML}

\vskip 0.3in
]

\printAffiliationsAndNotice{} 

\fi

\begin{abstract}
   Composition is a key feature of differential privacy. Well-known advanced composition theorems allow one to query a private database quadratically more times than basic privacy composition would permit. However, these results require that the privacy parameters of all algorithms be fixed \textit{before} interacting with the data. To address this, \citet{rogers2016odometer} introduced \textit{fully adaptive} composition, wherein both algorithms and their privacy parameters can be selected adaptively. They defined two probabilistic objects to measure privacy in adaptive composition: \textit{privacy filters}, which provide differential privacy guarantees for composed interactions, and \textit{privacy odometers}, time-uniform bounds on privacy loss. There are substantial gaps between advanced composition and existing filters and odometers. First, existing filters place stronger assumptions on the algorithms being composed. Second, these odometers and filters suffer from large constants, making them impractical. We construct filters that  match the rates of advanced composition, including constants, despite allowing for adaptively chosen privacy parameters. En route we also derive a privacy filter for approximate zCDP. We also construct several general families of odometers. These odometers match the tightness of advanced composition at an arbitrary, preselected point in time, or at all points in time simultaneously, up to a doubly-logarithmic factor. We obtain our results by leveraging advances in martingale concentration. In sum, we show that fully adaptive privacy is obtainable at almost no loss.
   %, and conjecture that our results are essentially unimprovable (even in constants) in general.
\end{abstract}
%\begin{keywords}%
%  Differential Privacy, Confidence Sequences, Martingale Concentration%
%\end{keywords}
\section{Introduction}
\label{sec:intro}

\emph{Differential privacy} \citep{dwork2006priv} is an algorithmic criterion that provides meaningful guarantees of individual privacy for analyzing sensitive data. Intuitively, an algorithm is differentially private if similar inputs induce similar distributions on outputs. More formally, an algorithm $A : \calX \rightarrow \calY$ is differentially private if, for any set of outcomes $G \subset \calY$ and any \textit{neighboring} inputs $x, x' \in \calX$,
\begin{equation}
    \label{eq:dp}
    \P(A(x) \in G) \leq e^\epsilon \P(A(x') \in G) + \delta,
\end{equation}
where $\epsilon$ and $\delta$ are the privacy parameters of the algorithm. 

A key property of differential privacy is graceful composition. Suppose $A_1, \dots, A_n$ are algorithms such that each $A_m$ is $(\epsilon_m, \delta_m)$-differentially private. Advanced composition~\citep{dwork2010boosting, kairouz2015composition} states that, for any $\delta' > 0$, the \textit{composed} sequence of algorithms is $(\epsilon, \delta)$-differentially private, where
$\delta = \delta' + \sum_{m \leq n}\delta_m$, and 
\begin{equation}
\label{eq:adv_comp}
\epsilon = \sqrt{2\log\left(\frac{1}{\delta'}\right)\sum_{m \leq n}\epsilon_m^2} + \sum_{m \leq n}\epsilon_m\left(\frac{e^{\epsilon_m} - 1}{e^{\epsilon_m} + 1}\right).
\end{equation}
When all privacy parameters are the same and small, we roughly have $\epsilon = O(\sqrt{n}\epsilon_m)$. Hence, analysts can make use of sensitive datasets with a slow degradation of privacy.

However, there is a major disconnect between most existing results on  privacy composition and modern data analysis. As analysts view the outputs of algorithms, the future manner in which they interact with the data changes. Advanced composition allows analysts to adaptively select algorithms, but not privacy parameters. In many cases, analysts may wish to choose the subsequent privacy parameters based on the outcomes of the previous private algorithms. For example, if an analyst learns, from past computations, that they only need to run one more computation, they should be able to use the remainder of their privacy budget in the final round. Likewise, if an analyst is having a hard time deriving conclusions, they should be allowed to adjust privacy parameters to extend the allowable number of computations.

This desideratum has motivated the study of \textit{fully adaptive} composition, wherein one is allowed to adaptively select the privacy parameters of the algorithms. \citet{rogers2016odometer} define two probabilistic objects which can be used to ensure privacy guarantees in fully adaptive composition. The first, called a \textit{privacy filter}, is an adaptive stopping condition that ensures an entire interaction between an analyst and a dataset retains a pre-specified target privacy level, even when the privacy parameters are chosen adaptively. The second, called a \textit{privacy odometer}, provides a sequence of high-probability upper bounds on how much privacy has been lost up to any point in time. While this work took the first steps towards fully adaptive composition, their filters and odometers suffered from large constants and the latter suffered from sub-optimal asymptotic rates.
%Since data analysts aim to keep the total level of privacy loss low for private computations, these large constants can make the privacy bounds from \cite{rogers2016odometer} practically unusable. Moreover, these privacy filters and odometers are designed to be valid when the algorithms being composed satisfy probabilistic (i.e. point-wise) differential privacy \cite{kasiviswanathan2014semantics}. In general, this is a \textit{stronger} assumption than differential privacy, and the authors must pay potentially large costs for converting from differential privacy to probabilistic privacy. 

We show that, as long as a target privacy level is pre-specified, one can obtain the same rate as advanced composition, including constants. We also construct families of privacy odometers that are not only tighter than the originals, but can be optimized for various target levels of privacy. Overall, we show that full adaptivity is not a cost---but rather a feature---of differential privacy.
% \swcomment{\jwcomment{is this addressed now?}I think this last part is the highlight of the introduction (and the paper). You definitely want to emphasize and expand this. You could start with "even though rogers et al., initiated the study of privacy filter and odometer, their bounds are ... " then give a quick pitch of what you have accomplished and say something like "our work significantly improves the bounds in rogers et al.... Our new results show that there is almost no additional loss in the privacy parameter due to adaptivity"}

\subsection{Related Work}

\paragraph{\textbf{Privacy Composition:}} 
There is a long line of work on privacy composition. The ``basic composition" theorem states that, when composing private algorithms, the privacy parameters (both $\epsilon$ and $\delta$) add up linearly \citep{dwork2006priv, dwork2006our, dwork2009differential}. 
The ``advanced composition" theorem allows the total $\epsilon$ to grow sublinearly with a small degradation on $\delta$ \citep{dwork2010boosting}.
Later work \citep{kairouz2015composition, murtagh2016complexity} studies ``optimal" composition, a computationally intractable formula that tightly characterizes the overall privacy of composed mechanisms.
% In a result known as \textit{basic} composition, \citet{dwork2006priv} show that the total privacy level of purely differentially private algorithms run back-to-back is bounded above by the sum of the privacy parameters of each algorithm. Later works \citep{dwork2006our, dwork2009differential} demonstrate that this result holds for \textit{approximate} differential privacy as well.
% \citet{dwork2010boosting} introduce advanced composition in the case of homogeneous privacy parameters through the lens of max-divergence.
% Later works \citep{kairouz2015composition, murtagh2016complexity} are able to improve upon advanced composition and extend the result to the case of heterogeneous privacy parameters. These authors are able to prove their results by ``reducing" general differentially private algorithms to instances of randomized response. These works also study optimal composition for differential privacy, a computationally intractable formula that tightly characterizes the overall privacy of composed interactions.

More recently, several variants of privacy have been studied including (zero)-concentrated differential privacy (zCDP)~\citep{bun2016concentrated, DworkR16}, Renyi differential privacy (RDP)~\citep{mironov2017renyi}, and $f$-differential privacy ($f$-DP)~\citep{dong2019gaussian}. These all exhibit tighter composition results than differential privacy, but for restricted classes of mechanisms. These results do not allow adaptive choices of privacy parameters.

\paragraph{Privacy Filters and Odometers:}
\citet{rogers2016odometer} originally introduced privacy filters and odometers, which allow privacy composition with adaptively selected privacy parameters. While their contributions provide a decent approximation of advanced composition, their bounds suffer from large constants, which prevents practical usage. Our work directly improves over these initial results. First, we construct privacy filters essentially matching advanced composition. We also provide flexible families of privacy odometers that outperform those of \citet{rogers2016odometer}. 

% Lastly, our filters and odometers can allow algorithms being composed to satisfy zCDP, yielding greater control over the privacy loss of certain private mechanisms. 

% \citet{rogers2016odometer} show that one can obtain roughly the same composition rate as advanced composition when privacy parameters are adaptively selected, albeit at the cost of large constants. They obtain their results by introducing tools known as privacy filters and odometers to bound privacy loss in composed interactions. 

\citet{feldman2020individual} leverage RDP to construct R\'enyi filters, where they require individual mechanisms to satisfy RDP. Since our proof establishes a new privacy filter for approximate zCDP \citep{bun2016concentrated}, our results also extend to approximate RDP \citep{Papernot022}, which directly generalizes their R\'enyi filter. Even though it is also possible to obtain a privacy filter for $(\epsilon, \delta)$-DP through R\'enyi filters \citep{feldman2020individual}, this result requires a stronger assumption that algorithms being composed satisfy \textit{probabilistic} (i.e.\! point-wise) differential privacy \citep{kasiviswanathan2014semantics}.
 Since converting from differential privacy to probabilistic differential privacy can be costly (see Lemma~\ref{lem:dp_to_pdp}), our filters demonstrate an improvement by avoiding the conversion cost. %Moreover, our proof technique does not rely on complicated under-the-hood conversions between various modes of privacy, adding simplicity to our approach. 

More recently,  \citet{KTH22} and \citet{SmithThakurta22} provide privacy filters for Gaussian DP (GDP) \citep{dong2019gaussian}. However,  their results do not hold for more general mechanisms under $f$-DP and therefore cannot handle algorithms with rare ``catastrophic" privacy failure events, in which the privacy loss goes to infinity. Both of our $(\epsilon, \delta)$-filter and approximate zCDP filters can handle such events.

\citet{feldman2020individual} and \citet{lecuyer2021practical} construct RDP odometers. The former work sequentially composes R\'enyi filters and the latter work simultaneously runs multiple R\'enyi filters and takes a union bound. Neither odometer provides high probability, time-uniform bounds on privacy loss, making these results incomparable to our own. We believe our notion of odometers, which aligns with that of \citet{rogers2016odometer}, is more natural.

%\citet{feldman2020individual} construct RDP odometers by sequentially composing R\'enyi filters, introducing potential looseness. A recent preprint \citep{lecuyer2021practical} presents additional RDP odometers, but we are unable to verify its correctness. This odometer has complicated parameters and leverages union bounds, likely increasing looseness. Neither RDP odometer provides high probability, time-uniform bounds on privacy loss, making these results incomparable to our own. We believe our notion of odometers, which aligns with that of \citet{rogers2016odometer}, is more natural.

To prove our results, we leverage time-uniform concentration results for martingales \citep{howard2020line, howard2021unif}. The bounds in these papers directly improve over related self-normalized concentration results \citep{de2004self, chen2014exponential}. These latter bounds were leveraged in \citet{rogers2016odometer} to construct filters and odometers.
\begin{figure*}[h!]
    \centering
    \subfloat[Comparing lower order terms]{
        \includegraphics[width=0.38\textwidth]{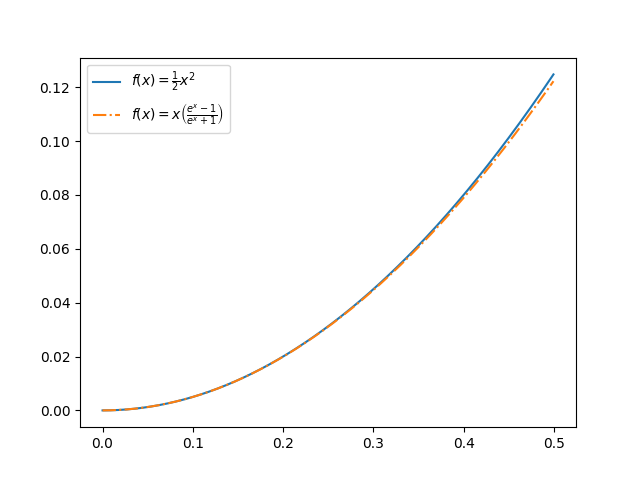}
        \label{fig:gap:gap}
    }
    \subfloat[Comparing privacy odometers]{
        \includegraphics[width=0.38\textwidth]{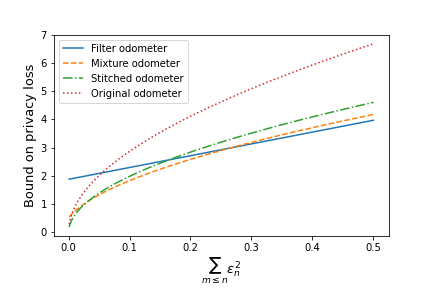}
        \label{fig:gap:odom}
    }
    \caption{Figure~\ref{fig:gap:gap} compares the lower order terms of advanced composition and our privacy filter. Figure~\ref{fig:gap:odom} compares the original odometer of~\citet{rogers2016odometer} with our odometers (filter, mixture, and stitched).
    }
    \label{fig:gap}
\end{figure*}
\subsection{Summary of Contributions}
\label{sec:intro:contribtions}

In this work, we provide two primary contributions. We present these results in full rigor following a brief discussion of privacy basics and martingale theory in Section~\ref{sec:background}.
% \begin{itemize}
%     \itemsep=0em
%     \item
    
    {\paragraph{Privacy Filters:}} In Theorem~\ref{thm:fully_adaptive} of Section~\ref{sec:filter}, we construct \textit{privacy filters} that match the rate of advanced composition, significantly improving over results of \citet{rogers2016odometer}. Our filter follows from a more general approximate zCDP/RDP filter \citep{bun2016concentrated, Papernot022} presented in Theorem~\ref{thm:azcdp}. In particular, this approximate zCDP/RDP filter greatly generalizes existing filters from the pure RDP setting \citep{feldman2020individual}. This extension allows us to capture a broader class of algorithms and avoids the conversion loss when translating bounds between pure RDP and $(\epsilon, \delta)$-differential privacy. We state an informal version of filter in the case of approximate differential privacy below\footnote{In Appendix \ref{sec:alterproof}, we provide an alternative proof for our privacy filter result through reductions to generalized randomized response. While it gives the exact same rates, we believe it could be of independent interest. For example, it may be useful for obtaining filters with rates like the optimal composition \citep{murtagh2016complexity, kairouz2015composition}, which used a similar reduction to randomized response in their analysis.}.
    \begin{informal}[Improved Privacy Filter]
\label{thm:informal}
Fix target privacy parameters $\epsilon > 0$ and $\delta > 0$, and suppose $(A_n)_{n \geq 1}$ is an adaptively selected sequence of algorithms. Assume that $A_n$ is $(\epsilon_n, \delta_n)$-DP conditioned on the outputs of the first $n - 1$ algorithms, where $\epsilon_n$ and $\delta_n$ may depend on outputs of $A_1, \dots, A_{n - 1}$.  If a data analyst stops interacting with the data before $\sqrt{2\log\left(\frac{1}{\delta}\right)\sum_{m \leq n + 1}\epsilon_m^2} + \frac{1}{2}\sum_{m \leq n + 1}\epsilon_m^2 > \epsilon$, then the entire interaction is $(\epsilon, \delta)$-DP.
%Fix $\epsilon > 0$ and $\delta = \delta' + \delta '' > 0$, and consider the stopping condition $N$ given as
%$$
%N := \inf\left\{n \in \N : \epsilon < \sqrt{2\log\left(\frac{1}{\delta'}\right)\sum_{m \leq n + 1}\epsilon_m^2} + \frac{1}{2}\sum_{m \leq n + 1}\epsilon_m^2\right\} \wedge \inf\left\{n \in \N : \delta'' <  \sum_{m \leq n + 1}\delta_m\right\},
%$$
%where $a \land b$ indicates the minimum of $a$ and $b$. Then, the composed algorithm $A_{1:N} = (A_1, \dots, A_N)$ is $(\epsilon, \delta)$-differentially private.

\end{informal}
    
    {\paragraph{Privacy Odometers:}} In Theorem~\ref{thm:new_odometers} of Section~\ref{sec:odometers}, we construct improved \textit{privacy odometers} --- that is, sequences of upper bounds on privacy loss which are all simultaneously valid with high probability. Our three families of odometers theoretically and empirically outperform those of \citet{rogers2016odometer}. See Figure~\ref{fig:gap:odom} for a comparison.

For both results, our key insight is to view adaptive privacy composition as depending not on the number of algorithms being composed, but rather on the sums of squares of privacy parameters, $\sum_{m \leq n}\epsilon_m^2$. This shift to looking at ``intrinsic time" allows us to apply recent advances in time-uniform concentration~\citep{howard2020line, howard2021unif} to privacy loss martingales. Overall, our results show that there is essentially no cost for fully adaptive private data analysis. 

\ifarXiv
\section{Background on Differential Privacy}
\label{sec:background}
Throughout, we assume all algorithms map from a space of datasets $\calX$ to outputs in a measurable space, typically either denoted $(\calY, \calG)$ or $(\calZ, \calH)$. For a sequence of algorithms $(A_n)_{n \geq 1}$, we often consider the composed algorithm $A_{1:n} := (A_1, \dots, A_n)$. For more background on measure-theoretic matters, as well as on the notion of neighboring datasets, see Appendix~\ref{app:notation}.

%\paragraph{Conditional Differential Privacy.}
We start by formalizing a generalization of differential privacy in which the privacy parameters of an algorithm $A_n$ can be functions of the outputs of $A_1, \dots, A_{n - 1}$. 
In particular, we replace the probabilities in Equation~\eqref{eq:dp} with conditional probabilities given relevant random variables. 

\begin{definition}[Conditional Differential Privacy]
\label{def:cdp}
Suppose $A$ and $B$ are algorithms mapping from a space $\calX$ to measurable spaces $(\calY, \calG)$ and $(\calZ, \calH)$ respectively. Suppose $\epsilon, \delta : \calZ \rightarrow \R_{\geq 0}$ are measurable functions. We say the algorithm $A$ is $(\epsilon, \delta)$-differentially private conditioned on $B$ if, for any neighbors $x, x' \in \calX$ and for all measurable sets $G \in \calG$, we have
\begin{align*}
&\P\left(A(x) \in G\mid B(x)\right) \\
&\;\;\leq e^{\epsilon(B(x))}\P\left(A(x') \in G \mid B(x)\right) + \delta(B(x)).
\end{align*}

For conciseness, we will write either $\epsilon$ or $\epsilon(x)$ for $  \epsilon(B(x))$ and likewise $\delta$ or $\delta(x)$ for $\delta(B(x))$.
\end{definition}

In the $n$th round of adaptive composition, we will set $A := A_n$ and $B := A_{1:n -1}$. In this setting, the analyst has functions $\epsilon_n, \delta_n :\calY^{n - 1} \rightarrow \R_{\geq 0}$ and takes the $n$th round privacy parameters to be $\epsilon_n(A_{1:n-1}(x))$ and $\delta_n(A_{1:n -1}(x))$. In other words, the analyst uses the outcome of the first $n - 1$ algorithms to decide the level of privacy for the $n$th algorithm, ensuring that $A_n$ is $(\epsilon_n, \delta_n)$-differentially private conditioned on $A_{1:n - 1}$.

We will also leverage the notion of \emph{zero-concentrated differential privacy (zCDP)} \citep{bun2016concentrated}, which often provides a cleaner analysis for privacy composition. First, we will recall the definition of R\'enyi divergence.

\begin{definition}
The R\'enyi divergence from $P$ to $Q$ of order $\lambda \geq 1$ is defined as
\[
D_\lambda(P \| Q) := \frac{1}{\lambda - 1} \log\left( \E_{Y\sim P} \left[ \left( \frac{P(Y)}{Q(Y)}\right)^{\lambda - 1} \right] \right).
\]
\end{definition}

\begin{comment}
\begin{definition}[Conditional R\'enyi divergence]
The conditional R\'enyi divergence between probability measures $P$ and $Q$ given a random variable $X$ is defined as
\[
D_\lambda^X(P \mid Q) := \frac{1}{\lambda - 1}\log\left(\E\left[\left(\frac{P(Y\mid X)}{Q(Y \mid X)}\right)^{\lambda - 1}\; \Big\vert\; X\right]\right),
\]
where the expectation above is with respect to the conditional distribution $P(Y \in \cdot \mid X)$ and $P(Y \mid X)/Q(Y \mid X)$ is shorthand for the ratio of conditional densities for $P$ and $Q$ given $X$.
    
\end{definition}
\end{comment}

The notion of zCDP bounds the R\'enyi divergence from $A(x)$ to $A(x')$ for any neighbors $x$ and $x'$. We will focus on a conditional version of a more general definition called approximate zCDP \citep{bun2016concentrated, Papernot022} that permits a small probability of unbounded R\'enyi divergence. The conditional approximate zCDP definition we provides uses the convex mixture formulation  adapted from \citet{Papernot022}, since it is more convenient for our proof. In Appendix \ref{sec:equiv}, we will show that in the case $\delta$ and $\rho$ are constant, this definition is equivalent to the original definition in \citet{bun2016concentrated}.

\begin{definition}[Conditional Approximate zCDP]\label{def:approxazcdp}
Suppose $A: \calX \times \calZ \to \calY$ with outputs in a measurable space $(\calY, \calG)$. Suppose $\delta, \rho: \calY \rightarrow \R_{\geq 0}$. We say the algorithm $A$ satisfies conditional $\delta(z)$-approximate  $\rho(z)$-zCDP if, for all $z \in \calZ$ and any neighboring datasets $x, x'$, there exist probability transition kernels\footnote{A \textit{probability transition kernel} $P' : \calZ \times \calG \rightarrow [0, 1]$ is a mapping such that $P(\cdot \mid z) : \calG \rightarrow [0, 1]$ is a probability measure for each $z \in \calZ$.} $P', P'', Q', Q'' : \calZ \times \calG \rightarrow [0, 1]$ such that the conditional outputs are distributed according to the following mixture distributions:
\begin{align*}
 &A(x; z)  \sim (1 - \delta(z))P'(\cdot \mid z) + \delta(z) P''(\cdot \mid z) \\
 &A(x'; z) \sim (1 - \delta(z))Q'(\cdot \mid z) + \delta(z) Q''(\cdot \mid z),
\end{align*}
where for all $\lambda\geq 1$, $D_\lambda(P'(\cdot \mid z) \| Q'(\cdot \mid z)) \leq \rho(z) \lambda$ and $D_\lambda(Q'(\cdot \mid z) \| P'(\cdot \mid z)) \leq \rho(z) \lambda$ for all $z \in \calZ$.
\end{definition}
\fi
We will also use the notions of filtration and martingales.

\paragraph{Filtration and Martingales:}
\label{sec:martingales:basics}
A process $(X_n)_{n \in \N}$ is said to be a martingale with respect to a filtration $(\calF_n)_{n \in \N}$ if, for all $n \in \N$, (a) $X_n$ is $\calF_n$-measurable, (b) $\E|X_n| < \infty$,  and (c) $\E(X_n \mid \calF_{n - 1}) = X_{n - 1}$. Correspondingly, $(X_n)_{n \in \N}$ is a supermartingale if $\E(X_n \mid \calF_{n - 1}) \leq X_{n - 1}$. In our context, we will consider the natural filtration $(\calF_n(x))_{n \in \N}$ generated by $(A_n(x))_{n \geq 1}$. In our proofs, we construct the appropriate (super)martingales so that we can leverage the optional stopping theorem and time-uniform concentration to obtain privacy filters and odometers \citep{ville1939etude, howard2020line, howard2021unif}. We present a full exposition of the mathematical tools in Appendix~\ref{app:notation} and \ref{app:concentration}.
\ifarXiv
\section{Privacy Filters}
\label{sec:filter}
We now provide our main results on privacy filter. In general, a privacy filter is a function $N$ that takes the privacy parameters of a sequence of private algorithms as input and decides to stop at some point so that the composition of these algorithms satisfies a pre-specified level of privacy. We will first present a privacy filter for approximate zCDP (Theorem \ref{thm:azcdp}), which will immediately imply the privacy filter result for $(\epsilon,\delta)$-DP (Theorem \ref{thm:fully_adaptive}). Since approximate zCDP bounds R\'enyi divergence of all orders $\lambda$, our proof for Theorem \ref{thm:azcdp} also directly implies a privacy fiter for approximate RDP \citep{Papernot022}, which generalizes the RDP filter by \citet{feldman2020individual}.

Our $(\epsilon, \delta)$-DP filter improves on the rate of the original filter presented in \citet{rogers2016odometer} and matches the rate of advanced composition that requires pre-fixed choices of privacy parameters. Even though it is also possible to obtain an $(\epsilon, \delta)$-DP filter through the result of \citet{feldman2020individual}, our privacy filters avoid their conversion costs and provide a tighter bound.\footnote{\citet[Section 4.3]{feldman2020individual} apply R\'enyi filters to algorithms which satisfy (conditional) probabilistic differential privacy (pDP). In general, a lossy conversion from $(\epsilon, \delta)$-DP to $(\epsilon, \delta)$-pDP is required to apply their filter.} 

% \swdelete{In addition, we improve over the filters of \citet{feldman2020individual} in that we assume the algorithms being composed satisfy \textit{conditional differential privacy}, whereas they assume \textit{conditional probabilistic differential privacy} when approximating advanced composition, otherwise needing to pay the potentially hefty conversion price outlined in Lemma~\ref{lem:dp_to_pdp}.}

\begin{comment}
However, we can prove the validity of our filter without under-the-hood conversions between differential privacy and R\'enyi differential privacy \cite{mironov2017renyi}, offering a more straightforward proof using the time-uniform bounds presented in \cite{howard2020line}.

\swdelete{
\begin{definition}[Privacy Filter \citep{rogers2016odometer}]
\label{def:filter}
Let $(A_n)_{n \geq 1}$ be an adaptive sequence of algorithms such that, for all $n \geq 1$, $A_n$ is $(\epsilon_n, \delta_n)$-DP conditioned on $A_{1:n - 1}$. Let $\epsilon > 0$ and $\delta > 0$ be target privacy parameters. Then, a function $N : \R_{\geq 0}^\infty \times \R_{\geq 0}^\infty \rightarrow \N$ is an $(\epsilon, \delta)$-privacy filter if 
\begin{enumerate}
    \itemsep=0em
    \item for all $x \in \calX$, $N(x) := N\left((\epsilon_n(x))_{n \geq 1}, (\delta_n(x))_{n \geq 1}\right)$ is a stopping time with respect to the natural filtration generated by $(A_n(x))_{n \geq 1}$, and
    \item the algorithm $A_{1 : N(\cdot)}(\cdot)$ is $(\epsilon, \delta)$-differentially private.
\end{enumerate}
\end{definition}}
\end{comment}

We can now state our general privacy filter in terms of approximate zCDP.  

\begin{thm}[Approximate zCDP filter]
\label{thm:azcdp}
Let $(A_n)_{n \geq 1}$ be an adaptive sequence of algorithms, where $A_n : \calX \times \calY^{n - 1} \rightarrow \calY$. Assume that $\delta_n,\rho_n : \calY^{n-1} \to \R_{\geq 0}$.
% meaning that they are are measurable functions $\calY^{n - 1} \rightarrow \mathbb{R}_{\geq 0}$. 
For any $n \geq 1$, assume that $A_n(\cdot; y_{1:n-1})$ is conditionally $\delta_n(y_{1:n-1})$-approximate $\rho_n(y_{1:n-1})$-zCDP for any prior outcomes $y_{1:n-1}$. We define the function $N: \calY^\infty \to \N$ where
\[
N(y_1, y_2, \cdots) = \inf \left\{n : \sum_{m = 1}^{n+1} \rho_\ell(y_{1:m-1}) > \rho\right\} \land \inf\left\{ n : \sum_{m = 1}^{n + 1} \delta_m(y_{1:m-1}) > \delta\right\}.
\]
Then $A_{1:N(\cdot)}(\cdot)$ is $\delta$-approximate $\rho$-zCDP, where $N(x) = N( (A_{n}(x))_{n \geq 1} ).$ 
%is a stopping time relative to the natural filtration generated by $A_1(x), A_2(x), \cdots$.
\end{thm}

We note that the argument used to prove the above theorem immediately implies a privacy filter for approximate RDP,  and thus Theorem~\ref{thm:azcdp} can be viewed as a strict generalization of the work of \citet{feldman2020individual}. Further, Theorem \ref{thm:azcdp} implies a privacy filter under $(\epsilon, \delta)$-differential privacy. To show this implication, we will use the following conversion results.
% due to \cite{bun2016concentrated}.

\begin{lem}[\citep{bun2016concentrated}]
If $A$ satisfies $(\epsilon, \delta)$-DP, then $A$ satisfies $\delta$-approximate $\frac{1}{2}\epsilon^2$-zCDP. If $A$ satisfies $\delta$-approximate $\rho$-zCDP, then $A$ satisfies $(\rho + 2\sqrt{\rho \ln(1/\delta')}, \delta + (1- \delta) \delta')$-DP.
\end{lem}

We can now obtain our $(\epsilon, \delta)$-privacy filter by a conversion of individual approximate differential privacy parameters to approximate zCDP ones, application of the approximate zCDP filter, and the conversion of approximate zCDP back to approximate differential privacy.

\begin{thm}[$(\epsilon, \delta)$-DP filter]
\label{thm:fully_adaptive} Suppose $(A_n)_{n \geq 1}$ is a sequence of algorithms such that, for any $n \geq 1$, $A_n$ is $(\epsilon_n, \delta_n)$-differentially private conditioned on $A_{1:n-1}$. Let $\epsilon > 0$  and $\delta = \delta' + \delta''$ be target privacy parameters such that $\delta' > 0, \delta'' \geq 0$ and for all outcomes $y = (y_1, y_2, \cdots)$ we have $\sum_{n=1}^\infty \delta_n(y_{1:n-1}) \leq \delta$. 
We define the function $N: \calY^\infty \to \N$ where
\[
N(y_1, y_2, \cdots) = \inf \left\{n : \sum_{m = 1}^{n+1} \epsilon^2_\ell(y_{1:m-1})/2 > \rho\right\} \land \inf\left\{ n : \sum_{m = 1}^{n + 1} \delta_m(y_{1:m - 1}) > \delta\right\}.
\]
Then, the algorithm $A_{1:N(\cdot)}(\cdot)$ is $(\rho + 2 \sqrt{\rho\log(1/\delta)}, \delta)$-DP, where $N(x) := N( (A_{n}(x))_{n \geq 1} ).$
\end{thm}

\begin{proof}[Proof of Theorem~\ref{thm:azcdp}]

In our proof, we assume that $\sum_{n = 1}^\infty \delta_n(y_{1:n - 1}) \leq \delta$ for all sequences $(y_n)_{n \geq 1}$ without loss of generality. 
%We will consider the unstopped process $A(x) = (A_1(x), A_2(x), \cdots)$.  Let $x, x'$ be neighbors, and let $P, Q$ denote the distributions of the observed outputs $A(x)$ when the inputs are $x, x'$ respectively. 
Let $P_{1:n}$ and $Q_{1:n}$ denote the joint distributions of $(A_1, \dots, A_n)$ with inputs $x$ and $x'$, respectively. We overload notation and write $P_{1:n}(y_1, \dots, y_n)$ and $Q_{1:n}(y_1, \dots, y_n)$ for the likelihood of $y_1, \dots, y_n$ under input $x$ and $x'$ respectively.  We similarly write $P_n(y_n \mid y_{1:n-1})$ and $Q_n(y_n \mid y_{1:n - 1})$ for the corresponding conditional densities.

By Bayes rule, for any $n \in \N,$ we have
\begin{align*}
P_{1:n}(y_{1}, \cdots, y_n) &= \prod_{m=1}^{n} P_m(y_m \mid y_{1:m-1}),\\
Q_{1:n}(y_{1}, \cdots, y_n) &= \prod_{m=1}^{n} Q_m(y_m \mid y_{1:m-1}).
\end{align*}

% It suffices to show that the two distributions can be decomposed as weighted mixtures of $P'$ and $P''$, and  $Q'$ and $Q''$ respectively such that the mixture weights on $P'$ and $Q'$ are at least $(1-\delta)$ and for all $\lambda \geq 1$,
% \begin{align}\label{eq:renyibound} 
%&\max\Big\{     D_\lambda\left(P' \| Q' \right), D_\lambda\left(Q' \| P' \right)\Big\} \leq \rho\lambda .
% \end{align}

By our assumption of approximate zCDP at each step $n$, we can write the conditional likelihoods of $P_n$ and $Q_n$ as the following convex combinations: 
\begin{align*}
P_n(y_n \mid y_{1:n-1}) &= (1 - \delta_n(y_{1:n-1})) P'_n(y_n \mid y_{1:n-1}) + \delta_n(y_{1:n-1}) P''_n(y_n \mid y_{1:n-1}),
\\
 Q_n(y_n \mid y_{1:n-1}) &= (1 - \delta_n(y_{1:n-1})) Q'_n(y_n \mid y_{1:n-1}) + \delta_n(y_{1:n-1}) Q''_n(y_n \mid y_{1:n-1}),
\end{align*}
such that for all $\lambda \geq 1$ and all prior outcomes $y_{1:n-1}$,
we have both 
{\small
\begin{align}
&D_\lambda\left(P_n'( \cdot \mid y_{1:n-1} ) ~\|~ Q_n'(\cdot \mid y_{1:n-1}) \right) \leq \rho_n(y_{1:n-1}) \lambda, \\
 &  D_\lambda\left(Q_n'( \cdot \mid y_{1:n-1} ) ~\|~ P_n'( \cdot \mid y_{1:n-1} ) \right)  \leq \rho_n(y_{1:n-1}) \lambda.
    \end{align}
}%

Now, from Lemma~\ref{lem:convex-combo-proof}, we can then write these distributions as a convex combination of ``good'' distributions for which R\'enyi divergence is small, and ``bad'' distributions for which the divergence may be unbounded. In more detail, using the assumption that $\sum_{n=1}^\infty \delta_n(y_{1:n-1}) \leq \delta$ for all seqeunces $(y_n)_{n \geq 1},$ we have, for all $n \geq 1$,
\begin{align}
\label{eq:P}  P_{1:n}(y_1, \cdots, y_n) & = (1-\delta) \underbrace{\prod_{m=1}^n P_m'(y_m | y_{1:m-1})}_{P_{1:n}'(y_1, \cdots, y_n )} + \delta P_{1:n}''(y_1, \cdots, y_n)
\\
\label{eq:Q} Q_{1:n}(y_1, \cdots, y_n) &= (1-\delta) \underbrace{\prod_{m=1}^n Q_m'(y_m | y_{1:m-1})}_{Q_{1:n}'(y_1, \cdots, y_n ) } + \delta Q_{1:n}''(y_1, \cdots, y_n).
\end{align}

From the above, if $N : \calY^\infty \rightarrow \N$ is the time outlined in the theorem statement, it follows that the joint densities\footnote{We ignore measure-theoretic concerns about specifying which dominating measures these densities are defined with respect to.} $P_{1:N}$ of $A_1(x), \cdots A_{N(x)}(x)$ and $Q_{1:N}$ of $A_1(x'), \cdots A_{N(x')}(x')$, and both can be written as a convex combination of distributions $(P_{1:N}', P_{1:N}'')$ and $(Q_{1:N}', Q_{1:N}'')$:
\begin{align*}
P_{1:N}(y_1, y_2, \cdots, y_{N}) & = (1-\delta) \underbrace{\prod_{n=1}^{N} P_n'(y_n | y_{1:n-1})}_{P'(y_1, y_2, \cdots, y_{N})} + \delta P_{1:N}''(y_1, y_2, \cdots, y_{N})
\\
Q_{1:N}(y_1, y_2, \cdots, y_{N}) &= (1-\delta) \underbrace{\prod_{n=1}^{N} Q_n'(y_n | y_{1:n-1})}_{Q'(y_1, y_2, \cdots, y_{N})} + \delta Q_{1:N}''(y_1, y_2, \cdots, y_{N})
\end{align*}
In the above, we notate quantities in terms of ``$N$'' instead of ``$N(x)$'' or ``$N(x')$'' since $N$ only depends on the underlying dataset $x$ or $x'$ \emph{through} the observed sequence of iterates $(y_n)_{n \geq 1}$.

%We will establish inequality \eqref{eq:renyibound}.
What remains now is to bound the R\'enyi divergence between $P_N'$ and $Q_N'$. We do this using an optional stopping argument for non-negative supermartingales (Lemma~\ref{fact:optstop}). Suppose $(Y_n')_{n \geq 1}$ is a process whose $n$th finite-dimensional distribution is given by $P_n'$.  For any fixed $\lambda \geq 1$, define the process $(M_n^{(\lambda)})_{n \geq 0}$ by:
\begin{comment}
X_k^{(\lambda)} &:= \sum_{n\leq k} \left\{ \log\left(\frac{P'_n( Y_n' \mid  Y_{1:n-1}' )}{Q'_n( Y_n' \mid  Y_{1:n-1}' ) } \right)  - \lambda \rho_n(Y_{1:n-1}')\right\},   \label{process1}
\end{comment}
\begin{align}
M_n^{(\lambda)} &:= \exp\left\{ (\lambda - 1) \sum_{m \leq n}\left[\log\left(\frac{P'_m( Y_m' \mid  Y_{1:m-1}' )}{Q'_m( Y_m' \mid  Y_{1:m-1}' )}\right) - \lambda\rho_m(Y_{1:m - 1}')\right] \right\}.   \label{process2}
\end{align}

It is clear that $M_n^{(\lambda)}$ is a non-negative supermartingale with respect to natural filtration $(\calF_n')_{n \geq 1}$ given by $\calF_n' := \sigma(Y_m' : m \leq n)$, a fact that we confirm in Lemma~\ref{claim:nsm}. We emphasize that $(\calF_n')_{n \geq 1}$ is not in fact the data generating filtration, but rather a tool used for theoretical analysis. In more detail, we consider this filtration because, heuristically, approximate zCDP aims at bounding the moment generating function of a ``good'' portion of the joint distribution --- the true joint distribution may allow some probability of catastrophic failure (i.e.\ unbounded privacy loss). We adopt the same convention that $N := N(y_1, y_2, \dots)$ with the explicit values of $(y_n)_{n \geq 1}$ clear from context. Observe that  $N((Y_n')_{n \geq 1})$ is a stopping time with respect to $(\calF_n')_{n \geq 0}$. We now invoke optional stopping (Lemma~\ref{fact:optstop}), which yields
\begin{align*}
\E[M_{N(Y_1', Y_2',\dots)}^{(\lambda)}] \leq 1 & \implies  \E\left[ \exp\left( (\lambda  -1) \sum_{n\leq N(Y_1', Y_2', \cdots)} \left\{ \log\left(\frac{P'_n( Y_n' \mid  Y_{1:n-1}' )}{Q'_n( Y_n' \mid  Y_{1:n-1}' ) } \right)  - \lambda \rho_n(Y_{1:n-1}')\right\} \right) \right] \leq 1 \\
& \implies  \E\left[ \exp\left( (\lambda  -1) \sum_{n\leq N(Y_1', Y_2', \cdots)} \log\left(\frac{P'_n( Y_n' \mid  Y_{1:n-1}' )}{Q'_n( Y_n' \mid  Y_{1:n-1}' ) } \right)   \right) \right] \leq e^{\lambda (\lambda  -1) \rho} \\
& \implies \E\left[ \exp\left( (\lambda  -1) \log\left(\frac{P'_{N}( Y_{1:N}'  )}{Q'_{N}( Y_{1:N}') } \right)   \right) \right] \leq e^{\lambda (\lambda  -1) \rho}.
\end{align*}
What we have just showed is precisely that
\[
D_{\lambda}\left(P_{N}' \mid Q_N'\right) \leq \rho\lambda,
\]
which is precisely the desired result. A symmetric argument yields an identical bound on $D_\lambda(Q_N' \mid P_N')$. Thus, we have showed the desired result.
\end{proof}

\fi
\section{Privacy Odometers}
\label{sec:odometers}
 Previously, we constructed privacy filters that matched the rate of advanced composition while allowing \textit{both} algorithms and privacy parameters to be chosen adaptively. While privacy filters require the total level of privacy to be fixed in advance, it is desirable to track the privacy loss at all steps without a pre-fixed budget~\citep{ligett2017accuracy}. We now study privacy odometers
which provide sequences of upper bounds on accumulated privacy loss that are valid at all points in time simultaneously with high probability. 

\subsection{Background on Privacy Loss and Odometers}
To formally introduce privacy odometers, we will first revisit the notion of \emph{privacy loss}, which measures how much information is revealed about the underlying input dataset. For neighbors $x, x' \in \calX$, let $p^x$ and $p^{x'}$ be the densities of $A(x)$ and $A(x')$ respectively. The privacy loss between $A(x)$ and $A(x')$ is defined as
\begin{equation}
\label{eq:ploss}
    \calL(x, x') := \log\left(\frac{p^x(A(x))}{p^{x'}(A(x))}\right).
\end{equation}
By Equation~\eqref{eq:ploss}, a negative privacy loss suggests that the input is more likely to be $x'$, and likewise a positive privacy loss  suggests that the input is more likely to be $x$. We now generalize privacy loss to its conditional counterpart.

\begin{definition}[Conditional Privacy Loss]
\label{def:cploss}
Suppose $A$ and $B$ are as in Definition~\ref{def:cdp}. Suppose $x , x' \in \calX$ are neighbors. Let $p^x(\cdot | \cdot), p^{x'}(\cdot | \cdot): \calY \times \calZ \rightarrow \R_{\geq 0}$ be conditional densities for $A(x)$ and $A(x')$ respectively given $B(x)$.\footnote{To ensure the existence of conditional densities, it suffices to assume that $\calY$ and $\calZ$ are \textit{Polish spaces} under some metrics $d_\calY$ and $d_\calZ$, and that $\calG$ and $\calH$ are the corresponding Borel $\sigma$-algebras associated with $d_\calY$ and $d_\calZ$ \citep{durrett2019probability}. These measurability assumptions are not  restrictive, as Euclidean spaces, countable spaces, and Cartesian products of the two satisfy these assumption.} The privacy loss between $A(x)$ and $A(x')$ conditioned on $B$ is given by
$$
\calL_{B}(x, x') := \log\left(\frac{p^x(A(x)| B(x))}{p^{x'}(A(x) | B(x))}\right).
$$

\end{definition}

Suppose $A_n$ is the $n$th algorithm being run and we have already observed $A_{1:n-1}(x)$ for some unknown input $x \in \calX$. If we are trying to guess whether $x$ or a neighbor $x'$ produced the data, we would consider the privacy loss between $A_n(x)$ and $A_n(x')$ conditioned on $A_{1 : n -1}(x)$. It is straightforward to characterize the privacy loss of a composed algorithm $A_{1:n}$ in terms of the privacy loss of each constituent algorithm $A_1, \cdots, A_n$. Namely, from Bayes rule, 
\begin{equation}
    \calL_{1:n}(x, x') = \sum_{m \leq n}\calL_{m}(x, x'), \label{eq:priv_decomp}
\end{equation}
where $\calL_m(x, x')$ is shorthand for the conditional privacy loss between $A_m(x)$ and $A_m(x')$ given $A_{1:m- 1}(x)$, per Definition~\ref{def:cploss}. Equation~\eqref{eq:priv_decomp} also holds at arbitrary random times $N(x)$ that only depend on the dataset $x \in \calX$ through observed algorithm outputs.

{The simple decomposition of privacy loss noted above motivates the study of an ``alternative", probabilistic definition of differential privacy.
Intuitively, an algorithm should be differentially private if, with high probability, the privacy loss is small. More formally, an algorithm $A: \calX \rightarrow \calY$ is said to be $(\epsilon, \delta)$-\textit{probabilistically differentially private}, or $(\epsilon, \delta)$-pDP for short, if, for all neighboring inputs $x, x' \in \calX$, we have $\P\left( |\calL(x, x')| > \epsilon \right) \leq \delta$. In the previous line (as well as in the remainder of the section), the randomness in $\calL(x, x')$ comes from the randomized algorithm $A$.

Unfortunately, as noted by \citet{kasiviswanathan2014semantics} (in which pDP is called \textit{point-wise indistinguishability}), pDP is a strictly stronger notion than DP. In particular, if an algorithm is $(\epsilon, \delta)$-pDP, it is also $(\epsilon, \delta)$-DP. The converse in general requires a costly conversion.

% We outline this in the following Lemma which can be found in \cite{kasiviswanathan2014semantics}.

\begin{lem}[Conversions between DP and pDP \citep{kasiviswanathan2014semantics}]
\label{lem:dp_to_pdp}
If $A$ is $(\epsilon, \delta)$-pDP, then $A$ is also $(\epsilon, \delta)$-DP. Conversely, if $A$ is $(\epsilon, \delta)$-DP, then $A$ is $(2\epsilon, \frac{2\delta}{\epsilon e^\epsilon})$-pDP.
\end{lem}

% \swdelete{In this work, we prove most results on the level of probabilistic differential privacy, and hence only ever use the cost-free forward conversion between the notions of privacy. When we do handle the case where the algorithms being composed are conditionally differentially private, we leverage an extension of the fact that differentially private algorithms can be viewed as post-processings of the randomized response mechanism \citep{kairouz2015composition} to avoid a conversion cost.}

We note that that \citet{guingona2023comparing} have recently shown that other possible conversion rates from probabilistic differential privacy to approximate differential privacy are possible. However, we note that these conversions require trading off tightness in the approximation parameter $\epsilon$ and the approximation parameter $\delta$. In particular, a fully tight conversion from probabilistic differenial privacy to approximate differential privacy is not possible.
We will work with the conditional counterpart of probabilistic differential privacy (pDP).

\begin{definition}[Conditional Probabilistic Differential Privacy]
\label{def:cpdp}
Suppose $A : \calX \rightarrow \calY$ and $B : \calX \rightarrow \calZ$ are algorithms, and $\epsilon, \delta : \calZ \rightarrow \R_{\geq 0}$ are measurable. Then,  $A$ is said to be $(\epsilon, \delta)$-\textit{probabilistically differentially private conditioned on $B$} if, for any neighbors $x, x' \in \calX$, we have $$\P\left(|\calL_{B}(x, x')| > \epsilon(B(x)) | B(x) \right) \leq \delta(B(x)).$$
\end{definition}
%\vspace*{-0.4\baselineskip}

}

While in Theorem~\ref{thm:fully_adaptive} we assumed that the algorithms being composed were \textit{conditionally differentially private}, here, we need to assume \textit{conditional probabilistic privacy}. This is because our goal is not differential privacy, but rather tight control over privacy loss. We conjecture that a version of our privacy odometer (in Theorem~\ref{thm:new_odometers}) that replaces pDP by DP and leaves all else identical does not hold. Our intuition for this conjecture is that there exist simple examples of algorithms satisfying $(\epsilon, \delta)$-DP that don't satisfy $(\epsilon, \delta)$-pDP (see Appendix~\ref{app:fail_dp}, for instance). We believe that, by sequentially composing such algorithms and using anti-concentration results, one can show that some odometers fail to be valid. We leave this as potential future work. In sequential composition, we would assume the $n$th  algorithm $A_n$ is $(\epsilon_n, \delta_n)$-pDP conditioned on $A_{1: n - 1}$. The privacy parameters would be given as functions of $A_{1:n-1}(x)$. Now we state the definition of privacy odometer, which provides bounds on privacy loss under arbitrary stopping conditions (e.g.\! conditions based on model accuracy).

\begin{definition}[Privacy Odometer \citep{rogers2016odometer}]
\label{def:odometer}
Let $(A_n)_{n \geq 1}$ be an adaptive sequence of algorithms such that, for all $n \geq 1$, $A_n$ is $(\epsilon_n, \delta_n)$-pDP conditioned on $A_{1:n - 1}$. Let $(u_n)_{n \geq 1}$ be a sequence of functions where $u_n : \R^{n - 1}_{\geq 0} \times \R^{n - 1}_{\geq 0} \rightarrow \R_{\geq 0}$. Let $\delta \in (0, 1)$ be a target confidence parameter. For $x \in \calX, n \geq 1$, define $U_n(x) := u_n(\epsilon_{1:n-1}(x), \delta_{1:n-1}(x))$. Then, $(u_n)_{n \geq 1}$ is called a $\delta$-privacy odometer if, for all $x, x' \in \calX$ neighbors, we have
$$
\P\left(\exists n \geq 1 : \calL_{1:n}(x, x') > U_n(x)\right) \leq \delta.
$$
\end{definition}

\subsection{Improved Privacy Odometers}
% \rrcomment{Proposed new section title: "Time-uniform Privacy Loss Bounds vis Improved Privacy Odometers}
We construct our privacy odometers in Theorem \ref{thm:new_odometers}. Our technical centerpiece is time-uniform concentration inequalities for martingales \citep{ville1939etude, howard2020line, howard2021unif}. For a martingale $(M_n)_{n \in \N}$ and confidence level $\delta > 0$, time-uniform concentration inequalities provides bounds $(U_n)_{n \in \N}$ satisfying $\P(\exists n \in \N : M_n > U_n) \leq \delta$. Thus, if we can create a martingale from privacy loss, we can use time-uniform concentration to construct odometers. Our proof first considers the case where each $A_n$ is $(\epsilon_n, 0)$-pDP and the  \emph{privacy loss martingale} $(M_n)_{n \in \N}$ ~\citep{dwork2010boosting} is given by $M_0 = 0$ and:
{\small
\begin{equation}\label{priv_martin}
    M_n := M_n(x, x') := \calL_{1:n}(x, x') - \sum_{m \leq n}\E\left(\calL_{m}(x, x')| \calF_{n - 1}(x)\right)
\end{equation}
}%
We then extend to the case of $\delta_n \geq 0$ via conditioning.

 To construct their filters and odometers, \citet{rogers2016odometer} use self-normalized concentration inequalities~\citep{de2004self, chen2014exponential}. We instead use advances in time-uniform martingale concentration \citep{howard2020line, howard2021unif}, which yields tighter results.
%and essentially unimprovable results under weaker assumptions. 

\begin{comment}
Given a privacy loss martingale $(M_n)_{n \in \N}$, we define the \textit{increments} of $(M_n)_{n \in \N}$ as the process $(\Delta M_n)_{n \geq 1}$ given by
$$
\Delta M_n := M_n - M_{n - 1} = \calL_n(x, x') - \E\left(\calL_n(x, x')|\calF_{n - 1}(x)\right).
$$
By the assumption of conditional $(\epsilon_n, 0)$-pDP, it is immediate that, for all $n \geq 1$, $\Delta M_n$ and $-\Delta M_n$ are $\epsilon_n^2$-subGaussian. Thus, we will directly be able to apply the various time-uniform bounds to the problems at hand. Without further ado, we now proceed to our presentation of our improved fully adaptive composition results.
\end{comment}

\begin{thm}
\label{thm:new_odometers}
Suppose $(A_n)_{n \geq 1}$ is a sequence of algorithms such that, for any $n \geq 1$, $A_n$ is $(\epsilon_n, \delta_n)$-pDP conditioned on $A_{1:n-1}$. Let $\delta = \delta' + \delta''$ be a target approximation parameter such that $\delta' > 0, \delta'' \geq 0$. Define $N := N((\delta_n)_{n \geq 1}) := \inf\left\{n \in \N : \delta'' < \sum_{m \leq n +1}\delta_m\right\}$ and $V_n := \sum_{m \leq n}\epsilon_m^2$. Define the following:
% Consider the following sequences of functions:
\begin{enumerate}
\itemsep=0em
\item{\textbf{Filter odometer.}} For any $\epsilon > 0$, let $y^\ast :=  \left(-\sqrt{2\log\left(\frac{1}{\delta'}\right)} + \sqrt{2\log\left(\frac{1}{\delta'}\right) + \epsilon}\right)^2$. Define  functions $(u_n^F)_{n \geq 1}$ by
\begin{align*}
u_n^F(\epsilon_{1:n}, \delta_{1:n}) := \begin{cases} \infty &n > N \\
\frac{\sqrt{2y^\ast\log\left(\frac{1}{\delta'}\right)}}{2} + \frac{\sqrt{2\log\left(\frac{1}{\delta'}\right)}}{2\sqrt{y^\ast}}V_n + \frac{1}{2}V_n &\text{otherwise}.
\end{cases}
\end{align*}
\item{\textbf{Mixture odometer.}} For any $\gamma >0$, define the sequence of functions $(u_n^M)_{n \geq 1}$ by
\begin{align*}
u_n^M(\epsilon_{1:n}, \delta_{1:n}) := \begin{cases} \infty &n > N\\
\sqrt{2\log\left(\frac{1}{\delta'}\sqrt{\frac{V_n + \gamma}{\gamma}}\right)\left(\gamma + V_n\right)} + \frac{1}{2}V_n &\text{otherwise}.
\end{cases}
\end{align*}
\item{\textbf{Stitched odometer.}} For any $v_0 > 0$, define the sequence of functions $(u_n^S)_{n \geq 1}$ by
{
\begin{align*}
u_n^S(\epsilon_{1:n}, \delta_{1:n}) :=\begin{cases} \infty &\hspace{-2cm}n > N \text{ or }  V_n < v_0\\
1.7\sqrt{V_n\left(\log\log\left(\frac{2V_n}{v_0}\right) + 0.72\log\left(\frac{5.2}{\delta'}\right)\right)} + \frac{1}{2}V_n &\hspace{-.2cm}\text{else}.
\end{cases}
\end{align*}
}%
\end{enumerate}
Then, any of the sequences $(u_n^F)_{n \geq 1}$, $(u_n^M)_{n \geq 1}$, or $(u_n^S)_{n \geq 1}$ is a $\delta$-privacy odometer.
\end{thm}

 The proof of Theorem \ref{thm:new_odometers} can be found in Appendix~\ref{app:proof}. We now provide intuition for our odometers, which are plotted in Figure~\ref{fig:odom_compare}. 
 Our insight is to view odometers not as functions of the number of algorithms being composed, but rather as functions of the intrinsic time $\sum_{m \leq n}\epsilon_m^2$. This reframing allows us to leverage the various time-uniform concentration inequalities discussed in Appendix~\ref{app:concentration}. The filter odometer is the tightest odometer when the value $\sum_{m \leq n}\epsilon_m^2$ is close to a fixed accumulated variance $y^\ast$, but the tightness drops off precipitously when $\sum_{m \leq n}\epsilon_m^2$ is far from $y^\ast$.
The mixture odometer, which is named after the \textit{the method of mixtures}~\citep{robbins1970statistical,delapena2007maximization,howard2021unif}, sacrifices tightness at any fixed point in time to obtain overall tighter bounds on privacy loss. This odometer can be numerically optimized, in terms of $\rho$, for tightness at a predetermined value $\sum_{m \leq n}\epsilon_m^2$. The stitched odometer, whose name derives from Theorem~\ref{fact:stitch}, is similarly tight across time. This odometer requires that $\sum_{m \leq n}\epsilon_m^2$ exceed some pre-selected ``variance" $v_0$ before becoming nontrivial (i.e. finite). Larger values of $v_0$ will yield tighter odometers, albeit at the cost of losing bound validity when accumulated variance is small. With this intuition, we can compare our odometers to the original presented in \citet{rogers2016odometer}.
%making the other presented odometers preferable in this case. This odometer is ideal when a data analyst may know approximately but not exactly the desired level of privacy. To see how this odometer looks when optimized for target levels  $\epsilon > 0$, see Figure~\ref{fig:odoms:filter}.

\begin{figure*}
    \centering
    \subfloat[Comparing filter odometers]{
        \includegraphics[width=0.33\textwidth]{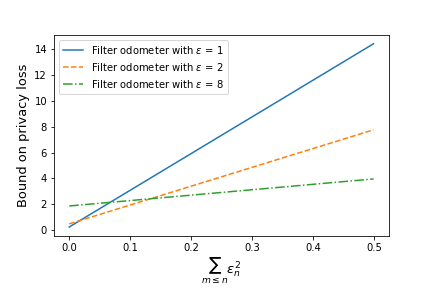}
        \label{fig:odoms:filter}
    }
    \subfloat[Comparing mixture odometers]{
        \includegraphics[width=0.33\textwidth]{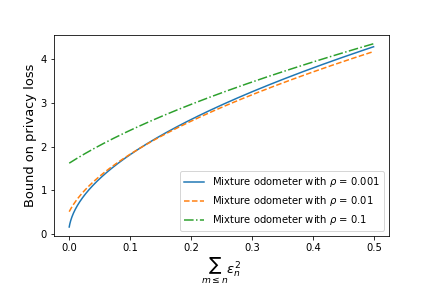}
        \label{fig:odoms:mixture}
    }
    \subfloat[Comparing stitched odometers]{
        \includegraphics[width=0.33\textwidth]{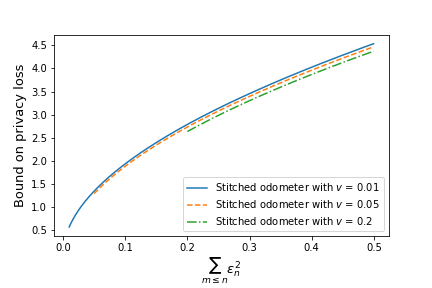}
        \label{fig:odoms:stitch}
    }
        
    \caption{Comparison of filter, mixture, and stitched odometers plotted as functions of $\sum_{m \leq n}\epsilon_m^2$. We set $\delta' = 10^{-6}$ and assume all algorithms being composed are purely differentially private for simplicity. }
    \label{fig:odoms}
\end{figure*}
\begin{figure*}[h!]
    \centering
    \subfloat[New odometers vs. original]{
        \includegraphics[width=0.35\textwidth]{figures/comp.png}
        \label{fig:odom_compare:rogers}
    }
    \subfloat[New odometers vs. pointwise advanced composition]{
        \includegraphics[width=0.35\textwidth]{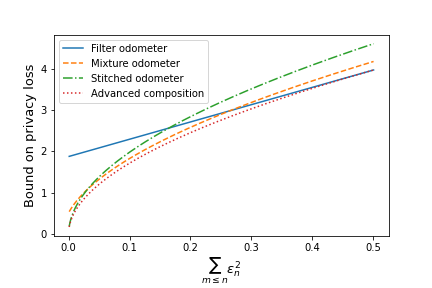}
        \label{fig:odom_compare:adv}
    }
    \caption{Figure~\ref{fig:odom_compare:rogers} compares our odometers to the original. Figure~\ref{fig:odom_compare:adv} compares them with advanced composition optimized point-wise. The curve plotted for advanced composition is valid at any fixed time, but not uniformly over time. Our odometers nevertheless provide a close approximation.}
    \label{fig:odom_compare}
    
\end{figure*}

\begin{comment}
\begin{figure}[htbp]
    \centering
    \begin{subfigure}{.32\textwidth}
        \centering
        \includegraphics[width=\textwidth]{figures/filter_comp.png}
        \caption{Comparing filter odometers}
        \label{fig:odoms:filter}
    \end{subfigure}
    \begin{subfigure}{.32\textwidth}
        \centering
        \includegraphics[width=\textwidth]{figures/mixture_comp.png}
        \caption{Comparing mixture odometers}
        \label{fig:odoms:mixture}
    \end{subfigure}
    \begin{subfigure}{.32\textwidth}
        \centering
        \includegraphics[width=\textwidth]{figures/stitch_comp.png}
        \caption{Comparing stitched odometers}
        \label{fig:odoms:stitch}
    \end{subfigure}

    \caption{Comparison of filter, mixture, and stitched odometers for various values of free parameters. We set $\delta' = 10^{-6}$ and assume all algorithms being composed are purely differentially private for simplicity. The $x$-axis shows various potential values of $\sum_{m \leq n}\epsilon_m^2$, and the $y$-axis shows the corresponding bound on privacy loss. Note that the stitched odometer is only valid when $\sum_{m \leq n}\epsilon_m^2 > v$, illustrating a tradeoff between tightness and validity. }
    \label{fig:odoms}
\end{figure}
\end{comment}

%However, since the privacy level $\epsilon_1$ for the first algorithm is typically known before querying the private dataset, we can always preselect $v_0 \geq \epsilon_1^2$.

\begin{lem}[Theorem 6.5 in \citet{rogers2016odometer}]
\label{fact:rogers_odometers}
Assume the same setup as Theorem~\ref{thm:new_odometers}, and fix $\delta = \delta' + \delta''$, where $\frac{1}{e} \geq \delta' > 0$ and $\delta'' \geq 0$. Define the sequence of functions $(u_n^R)_{n \geq 1}$ by
{
\begin{align*}
u_n^R(\epsilon_{1:n}, \delta_{1:n}) := \begin{cases} \infty, &n > N\\
\sqrt{2V_n\left( \log(110e) + 2\log\left(\frac{\log(|x|)}{\delta'}\right)\right)} + \frac{1}{2}V_n &\hspace{-2cm}n \leq N,  V_n \in \left[\frac{1}{|x|^2}, 1\right] \\
\sqrt{2\left(\frac{1}{|x|^2} + V_n\right)\left(1 + \frac{1}{2}\log\left(1 + |x|^2V_n\right)\right)\log\log\left(\frac{4}{\delta'}\log_2(|x|)\right)} & + \frac{1}{2}V_n, \\ &\text{otherwise}
\end{cases},
\end{align*}
}%
where $|x|$ denotes the number of elements in dataset $x$. Then, $(u_n^R)_{n \geq 1}$ is a $\delta$-privacy odometer.

\end{lem}
Our new odometers improve over the one presented in Lemma~\ref{fact:rogers_odometers}. First, the above odometer has an explicit dependence on dataset size. In learning settings, datasets are large, degrading the quality of the odometer. Secondly, the tightness of the odometer drops off outside of the interval $\left[ \frac{1}{|x|^2}, 1\right]$. If \textit{any} privacy parameter of an algorithm being composed exceeds $1$, the bound becomes significantly looser. Lastly, and perhaps most simply, the form of the odometer is complicated. Our odometers all have relatively straightforward dependence on the intrinsic time $\sum_{m \leq n}\epsilon_m^2$.

We now examine the rates of all odometers. For simplicity, let $v := \sum_{m \leq n}\epsilon_m^2$. The stitched odometer has a rate of $O(\sqrt{v\log\log(v)})$ in its leading term, asymptotically matching the law of the iterated logarithm \citep{robbins1970statistical} up to constants. Both the original privacy odometer and the mixture odometer have a rate of $O(\sqrt{v\log\left(v\right)})$, demonstrating worse asymptotic performance. The filter odometer has the worst asymptotic performance, growing linearly as $O\left(v\right)$. This does not mean the stitched odometer is the best odometer, since target levels of privacy are often kept small.

To empirically compare odometers, it suffices to consider the setting of \textit{pure} differential privacy, as the odometers identically depend on $(\delta_n)_{n \geq 1}$. Each presented odometer can be viewed as a function of $v$, allowing us to compare odometers by plotting their values for a continuum of $v$. 
Figure~\ref{fig:odom_compare:rogers} shows that there is no clearly tightest odometer. All odometers, barring the original, dominate for some window of values of $v$. While the stitched odometer is asymptotically best, the mixture odometer is tighter for small values of $v$. Likewise, if one knows an approximate target privacy level, the filter odometer is tightest. This behavior is expected from our understanding of martingale concentration~\citep{howard2020line,howard2021unif}: there is no uniformly tightest boundary containing (with probability $1-\delta$) the entire path of a martingale; boundaries that are tight early must be looser later, and vice versa. In fact, we conjecture that our bounds are essentially unimprovable in general --- this conjecture stems from the fact that the time-uniform martingale boundaries employed have error probability \emph{essentially} equal to $\delta$, which in turn stems from the deep fact that for continuous-path (and thus continuous-time) martingales, Ville's inequality (Fact \ref{fact:ville})---that underlies the derivation of these boundaries---holds with exact equality. Since we operate in discrete-time, the only looseness in Ville's inequality stems from lower-order terms that reflect the possibility that at the stopping time, the value of the stopped martingale may not be \emph{exactly} the value at the boundary. 

\begin{comment}
Lastly, several of the odometers (e.g. mixture and filter odometers) can, in a sense, be optimized for various points in time. As such, to do a proper comparison, in a given graphic, we always optimize said odometers for the same point in time, as we feel this gives the best understanding of how the odometers globally compare.
\end{comment}

\begin{comment}
\begin{figure}[h!]
    \centering
    \begin{subfigure}{0.48\textwidth}
        \centering
        \includegraphics[width=\textwidth]{figures/comp.png}
        \caption{New odometers vs. original}
        \label{fig:odom_compare:rogers}
    \end{subfigure}
    \begin{subfigure}{0.48\textwidth}
        \centering
        \includegraphics[width=\textwidth]{figures/comp_adv.png}
        \caption{New odometers vs. pointwise advanced composition}
        \label{fig:odom_compare:adv}
    
    \end{subfigure}
    
    \caption{Figure~\ref{fig:odom_compare:rogers} compares the three novel odometers to the original as functions of $\sum_{m \leq n}\epsilon_m^2$. The filter odometer is set at $\epsilon = 8$, the mixture odometer is set at $\rho = 0.01$, and the stitched odometer is taken with variance term $v = 0.005$. Figure~\ref{fig:odom_compare:adv} compares the three novel odometers with advanced composition optimized pointwise. The curve plotted for advanced composition \textbf{is not} a valid privacy odometer, but we see our valid odometers provide a very close overall approximation.}
    \label{fig:odom_compare}
\end{figure}
\end{comment}

In Figure~\ref{fig:odom_compare:adv}, we compare our odometers with advanced composition optimized in a point-wise sense for all values of $v$ simultaneously. This boundary \textit{is not a valid odometer}, as advanced composition only holds at a prespecified point in intrinsic time $v$. Our odometers are almost tight with advanced composition for the values of $v$ plotted. Our filter odometer lies tangent to the advanced composition curve, as expected from Section 5.2 of~\citet{howard2020line}. %Figure~\ref{fig:odom_compare} shows one can obtain tight, time-uniform control over privacy loss.

% \section{Summary of Results and Future Directions}
\section{Future Directions}
\label{sec:conc}

% In this work, we showed that one can match the rate of advanced composition while adaptively selecting both privacy parameters and algorithms.
% We constructed privacy filters that improve on both the original fully adaptive composition results presented in \cite{rogers2016odometer} and the more recent results presented in \cite{feldman2020individual}. We also constructed families of odometers that greatly outperform those presented in \cite{rogers2016odometer}. Our insight is to view composition in terms of the ``intrinsic time" $\sum_{m \leq n}\epsilon_m^2$. This allows us to leverage breakthroughs in self-normalized, time-uniform concentration \citep{howard2020line, howard2021unif}. Our results show that full adaptivity is in fact a feature of differential privacy. Informally, advanced composition can simply be seen as a consequence of time-uniform martingale concentration.

There are many open problems related to fully adaptive composition. For example, even though privacy filters have been studied under the notion of Gaussian DP~\citep{SmithThakurta22, KTH22}, privacy filters and odometers have not been studied for general $f$-DP \citep{dong2019gaussian}. It also has not been investigated whether adaptivity in privacy parameter selection improves the performance of iterative algorithms such as private SGD. Intuitively, it should be beneficial to let the iterates of an algorithm guide future choices of privacy parameters. Optimal composition results \citep{kairouz2015composition,murtagh2016complexity,zhu2021optimal} have yet to be considered in a setting where privacy parameters are adaptively selected. In Appendix~\ref{sec:alterproof}, we provide another proof of Theorem~\ref{thm:fully_adaptive}, which leverages a reduction of private algorithms to generalized randomized response. Since such a reduction was used in the proofs of \citet{kairouz2015composition} and \citet{murtagh2016complexity}, we believe this proof can be useful for optimal composition with adaptively chosen privacy parameters.

% \rrcomment{Add here that the alternative filters proof with pDP to DP filters might help with this analysis?}. 
%Generalizing these results could allow for tight control over privacy loss in more flexible, thus practical, settings.

% The most important direction for fully adaptive privacy composition pertains to applications to subsampled mechanisms \citep{wang2019subsampled}. Privacy filters can be applied to subsampled methods when a target level of privacy in known. However, subsampled mechanisms are often used in settings where criteria such as accuracy define stopping conditions~\citep{abadi2016deep}. Existing R\'enyi odometers \citep{feldman2020individual, lecuyer2021practical} do not provide high probability, time-uniform bounds on privacy loss.  We believe the machinery introduced in this paper can be used in this setting to construct tight RDP odometers (in the sense of \citet{rogers2016odometer}). 
\subsection*{Acknowledgements}

AR acknowledges support from NSF DMS 1916320 and an ARL IoBT CRA grant. Research reported in this paper was sponsored in part by the DEVCOM Army Research Laboratory under Cooperative Agreement W911NF-17-2-0196 (ARL IoBT CRA). The views and conclusions contained in this document are those of the authors and should not be interpreted as representing the official policies, either expressed or implied, of the Army Research Laboratory or the U.S. Government. The U.S. Government is authorized to reproduce and distribute reprints for Government purposes notwithstanding any copyright notation herein. ZSW and JW were supported in part by the NSF CNS2120667, NSF Award \#2120667, a CyLab 2021 grant, a Google Faculty Research Award, and a Mozilla Research Grant. JW acknowledges support from NSF GRFP grants DGE1745016 and DGE2140739.

\bibliography{bib.bib}{}
\bibliographystyle{plainnat}
% \pagebreak
\appendix
\section{Measure-Theoretic Formalism}
\label{app:notation}

Below, we provide some measure-theoretic formalisms and details regarding datasets and neighboring relations.

\paragraph{Neighboring Datasets:} 
Roughly speaking, an algorithm is differentially private if it difficult to distinguish between output distributions when the algorithm is run on similar inputs. In general, this notion of similarity amongst inputs is defined as a \textit{neighboring relation} $\sim$ between elements on the input space $\calX$. In particular, if two inputs (also referred to as datasets or databases) $x, x' \in \calX$ satisfy the neighboring relation $x \sim x'$, the we say $x$ and $x'$ are \textit{neighbors}. 

There are several canonical examples of neighboring relations on the space of inputs $\calX$. One example is where $\calX = \mathbb{X}^n$ for some data domain $\mathbb{X}$. The data domain can be viewed as the set of all possible individual entries for a dataset, and the space $\mathbb{X}^n$ correspondingly contains all possible $n$ element datasets. In this setting, databases $x, x' \in \calX$ may be considered neighbors if $x$ and $x'$ differ in exactly one entry. Another slightly more general setting is when $\calX = \mathbb{X}^*$, i.e., all possible datasets of finite size. In this situation, the earlier notion of neighboring still makes sense. However, in addition, we may say input datasets $x$ and $x'$ are neighbors if $x$ can be obtained from $x'$ by either adding or deleting an element. This is a very natural notion of neighboring, as under such a relation an algorithm would be differentially private if it were difficult to determine the presence or absence of an individual. Our work is agnostic to the precise choice of neighboring relation. As such, we choose to leave the notion as general as possible.

\paragraph{Algorithms and Random Variables:} We will consider algorithms as randomized mappings $A : \calX \rightarrow \calY$ taking inputs from $\calX$ to some output space $\calY$. To be fully formal, we consider the output space $\calY$ as a \textit{measurable space} $(\calY, \calG)$, where $\calG$ is some $\sigma$-algebra denoting possible events. Recall that a $\sigma$-algebra $\calS$ for a set $S$ is simply a subset of $2^S$ containing $S$ and $\emptyset$ that is closed under countable union, intersection, and complements. When we say $A$ is an algorithm having inputs in some space $\calX$, we really mean $A(x)$ is a $\calY$-valued random variable for any $x \in \calX$. The space $\calX$ need not have an associated $\sigma$-algebra, as algorithm inputs are essentially just indexing devices. Given a sequence of algorithms $(A_n)_{n \geq 1}$,  $(A_n(x))_{n \geq 1}$ is a sequence of $\calY$-valued random variables, for any $x \in \calX$.\footnote{Even if algorithms have different types of outputs (maybe some algorithms have categorical outputs while others output real-valued vectors), $\calY$ can still be made appropriately large to contain all possible outcomes.}

Since we are dealing with the composition of algorithms, we write $A_{1:n}(x)$ as shorthand for the random vector of the first $n$ algorithm outputs, i.e. $A_{1:n}(x) = (A_1(x), \dots, A_n(x))$. Formally, the random vector $A_{1:n}(x)$ takes output values in the product measurable space $(\calY^n, \calG^{\otimes n})$ where $\calG^{\otimes n}$ denotes the $n$-fold product $\sigma$-algebra of $\calG$ with itself. Likewise, since the number of algorithm outputs one views in fully-adaptive composition may be random, if $N$ is a random time (i.e. a $\N$-valued random variable), we will often consider the random vector $A_{1:N}(x) = (A_1(x), \dots, A_N(x))$. 

\iffalse We can view $A_{1:N}(x)$ as a random vector in the infinite product output space $(\calY^\infty, \calG^{\otimes \infty})$. We do this by simply inserting a special terminating element $\square$ into $\calY$ and repeating it ad infinitum after the last element of the random vector, i.e. $A_{1:N}(x) = (A_1(x), \dots, A_N(x), \square, \square, \dots)$. The reader should not concern themselves with the detail here --- we just include this information to be fully formal and so it is explicit which sets of outcomes are considered measurable.\fi

\paragraph{Filtrations and Stopping Times:} Since privacy composition involves sequences of random outputs, we will use the measure-theoretic notion of a \textit{filtration}. If we have fixed an input $x \in \calX$, we can assume the random sequence $(A_n(x))_{n \geq 1}$ is defined on some probability space $(\Omega, \calF, \P)$. Given such a probability space, a filtration $(\calF_n)_{n \in \N}$ of $\calF$ is a sequence of $\sigma$-algebras satisfying: (i) $\calF_n \subset \calF_{n + 1}$ for all $n \in \N$, and (ii) $\calF_n \subset \calF$ for all $n \in \N$. Given an arbitrary $\calY$-valued discrete-time stochastic process $(X_n)_{n \geq 1}$, it is often useful to consider the \textit{natural filtration} $(\calF_n)_{n \in \N}$ given by $\calF_n := \sigma(X_m : m\leq n)$ and $\calF_0 = \{\emptyset, \Omega\}$. Intuitively, a filtration formalizes the notion of accumulating information over time. In particular, in the context of the natural filtration generated by a stochastic process, the $n$th $\sigma$-algebra in the filtration $\calF_n$ essentially represents the entirety of information contained in the first $n$ random variables. In other words, if one is given $\calF_n$, they would know all possible events/outcomes that could have occurred up to and including timestep $n$.

Lastly, we briefly mention the notion of a \textit{stopping time}, as this measure-theoretic object is necessary to define privacy filters. Given a filtration $(\calF_n)_{n \in \N}$, a random time $N$ is said to be a stopping time with respect to $(\calF_n)_{n \in \N}$ if, for any $n$, the event $\{N \leq n\} \in \calF_n$. In words, a random time $N$ is a stopping time if given the information in $\calF_n$ we can determine whether or not we should have stopped by time $n$. Stopping times are essential to the study of fully-adaptive composition, as a practitioner of privacy will need to use the adaptively selected privacy parameters to determine whether or not to stop interacting with the underlying sensitive database. 

\section{Martingale Inequalities}
\label{app:concentration}

In this appendix, we provide a thorough exposition into the concentration inequalities leveraged in this paper. First, at the heart of supermartingale concentration is Ville's inequality \citep{ville1939etude}, which can be viewed as a time-uniform version of Markov's inequality.

\begin{lem}[Ville's Inequality \citep{ville1939etude}]
\label{fact:ville}
Let $(X_n)_{n \in \N}$ be a nonnegative supermartingale with respect to some filtration $(\calF_n)_{n \in \N}$. Then, for any confidence parameter $\delta \in (0, 1)$, we have 
\(
\P\left(\exists n \in \N : X_n \geq \frac{\E X_0}{\delta}\right) \leq \delta.
\)
\end{lem}

We do not directly leverage Ville's inequality in this work, but all inequalities we use can be directly proven from Lemma~\ref{fact:ville} \citep{howard2020line, howard2021unif}. In short, each inequality in this supplement is proved by carefully massaging a martingale of interest into a non-negative supermartingale.

Another useful tool we will leverage is Doob's optional stopping theorem.

\begin{lem}[Optional stopping theorem \citep{DU04}]
\label{fact:optstop}
Let $(X_n)_{n \in \N}$ be a nonnegative supermartingale with respect to some filtration $(\calF_n)_{n \in \N}$. Then
$\E\left[ X_\tau \right]  \leq \E\left[ X_0 \right]$ for all stopping times $\tau$ that are potentially infinite.
\end{lem}

For our alternative proof of the privacy filter (in Section \ref{sec:alterproof}), we leverage the following special case of a recent advance in time-uniform martingale concentration~\citep{howard2020line}. The following Theorem~\ref{fact:line_cross} is just a special case of the main result in \citet{howard2020line}, and we include the proof for completeness. When we say a random variable $X$ is $\sigma^2$-subGaussian conditioned on some sigma-algebra $\calG$, we mean that, for all $\lambda \geq 0$,
$$
\E\left(e^{\lambda X} \mid \calG\right) \leq e^{\lambda^2\sigma^2/2}.
$$
In particular, if $X$ is $\sigma^2$-subGaussian as above, this does not  imply that $-X$ is $\sigma$-subGaussian (because the condition is only assumed for $\lambda \geq 0$). In general, $X$ can have different behaviors in its left and right tail, see for example the discussion of the differing tails of the empirical variance of Gaussians in~\citet{howard2021unif}.

\begin{thm}
\label{fact:line_cross}
Let $(M_n)_{n \in \N}$ be a martingale with respect to some filtration $(\calF_n)_{n \in \N}$ such that $M_0 = 0$ almost surely. Moreover, let $(\sigma_n)_{n \geq 1}$ be a $(\calF_n)_{n \in \N}$-predictable sequence of random variables such that, conditioned on $\calF_{n - 1}$, $\Delta M_n := M_n - M_{n - 1}$ is $\sigma_n^2$-subGaussian.
Define $V_n := \sum_{m \leq n}\sigma_m^2$. Then, we have, for all $a, b > 0$, 
$$
\P\left(\exists n \in \N : M_n \geq \frac{b}{2} + \frac{b}{2a}V_n\right) \leq \exp\left(\frac{-b^2}{2a}\right).
$$

\end{thm}

\begin{proof}[\textbf{Proof of Theorem~\ref{fact:line_cross}}]
Let $(M_n)_{n \in \N}$ be the martingale listed in the theorem statement. Observe that, for any $a, b > 0$, the process $(X_n)_{n \in \N}$ given by
$$
X_n := \exp\left(\frac{b}{a}M_n - \frac{b^2}{2a^2}\sum_{m \leq n}\sigma_m^2\right)
$$
is a non-negative supermartingale. As such, applying Ville's inequality (Lemma~\ref{fact:ville}) yields
$$
\P\left(\exists n \in \N : X_n > \exp\left(\frac{b^2}{2a}\right)\right) \leq \exp\left(-\frac{b^2}{2a}\right).
$$
Now, on such event, taking logs and rearranging yields
$$
\frac{b}{a}M_n \leq \frac{b^2}{2a} + \frac{b^2}{2a^2}\sum_{m \leq n}\sigma_m^2.
$$
Multiplying both sides by $\frac{a}{b}$ finishes the proof.\end{proof}

The predictable process $(V_n)_{n \in \N}$ is a proxy for the accumulated variance of $(M_n)_{n \in \N}$ up to any fixed point in time. In particular, the process $(V_n)_{n \in \N}$ can be thought of as yielding the ``intrinsic time" of the process. The free parameters $a$ and $b$ thus allow us to optimize the tightness of the boundary for some intrinsic moment in time. This is ideal for us, as, for the sake of composition, the target privacy parameter $\epsilon$ can guide us in finding a point in intrinsic time (that is, in terms of the process $(V_n)_{n \in \N}$) to optimize for. We discuss how to apply this inequality to prove privacy composition results both in this supplement and in Section~\ref{sec:filter}.

We also leverage the following martingale inequalities from \citet{howard2021unif} in Section~\ref{sec:odometers}, where we construct various families of time-uniform bounds on privacy loss in fully-adaptive composition. These inequalities take on a more complicated form than Theorem~\ref{fact:line_cross}, but we explain the intuition behind them in the sequel. The first bound we present relies on the method of mixtures for martingale concentration, which stems back to Robbins' work in the 1970s \citep{robbins1970statistical}. There are many good resources providing an introduction to the method of mixtures \citep{delapena2007maximization, kaufmann2021mixture, howard2021unif}.

\begin{thm}
\label{fact:mixture}
Let $(M_n)_{n \in \N}$ be a martingale with respect to some filtration $(\calF_n)_{n \in \N}$ such that $M_0 = 0$ almost surely. Moreover, let $(\sigma_n)_{n \geq 1}$ be a $(\calF_n)_{n \in \N}$-predictable sequence of random variables such that, conditioned on $\calF_{n - 1}$, $\Delta M_n := M_n - M_{n - 1}$ is $\sigma_n^2$-subGaussian.
Define $V_n := \sum_{m \leq n}\sigma_m^2$ and choose a tuning parameter $\gamma > 0$. Then, for any $\delta > 0$, we have
$$
\P\left(\exists n \in \N : M_n \geq \sqrt{2(V_n + \gamma)\log\left(\frac{1}{\delta}\sqrt{\frac{V_n + \gamma}{\gamma}}\right)}\right) \leq \delta.
$$

\end{thm}

The next inequality relies on the recent technique of boundary stitching, first presented in \citet{howard2021unif}. Intuitively, the technique works by breaking intrinsic time --- that is, time according to the accumulated variance process $(V_n)_{n \in \N}$ --- into roughly geometrically spaced pieces. Then, one optimizes a tight-boundary in each region and takes a union bound. The actual details are more technical, but are not needed in this work. 

\begin{thm}
\label{fact:stitch}
Let $(M_n)_{n \in \N}$ be a martingale with respect to $(\calF_n)_{n \in \N}$ such that $M_0 = 0$ almost surely. Moreover, let $(\sigma_n)_{n \geq 1}$ be a $(\calF_n)_{n \in \N}$-predictable sequence of random variables such that, conditioned on $\calF_{n - 1}$, both $\Delta M_n := M_n - M_{n - 1}$ and $-\Delta M_n$ are $\sigma_n^2$-subGaussian.
Define $V_n := \sum_{m \leq n}\sigma_m^2$ and choose a starting intrinsic time $v_0 > 0$. Then, for any $\delta \in (0, 1)$, we have
{
\begin{align*}
\P\Bigg(\exists n \in \N : M_n \geq 1.7\sqrt{V_n\left(\log\log\left(\frac{2V_n}{v_0}\right) + .72\log\left(\frac{5.2}{\delta}\right)\right)} 
\quad\text{and} \quad V_n \geq v_0\Bigg) \leq \delta.
\end{align*}
}%
\end{thm}

Note that the original version of Theorem~\ref{fact:stitch} as found in \citet{howard2021unif} has more free parameters to optimize over, but we have already simplified the expression to make the result more readable. The free parameter $v_0 > 0$ in the above boundary gives the intrinsic time at which the boundary becomes non-trivial (i.e., the tightest available upper bound before $V_n \geq v_0$ is $\infty$). 

We qualitatively compare these bounds in Section~\ref{sec:odometers}, wherein we construct various time-uniform bounds on privacy loss processes. For now, Theorem~\ref{fact:line_cross} can be thought of as providing a tight upper bound on a martingale at a single point in intrinsic time, providing loose guarantees elsewhere. On the other hand, Theorems~\ref{fact:mixture} and \ref{fact:stitch} provide decently tight control over a martingale at all points in intrinsic time simultaneously, although at the cost of sacrificing tightness at any given fixed point.

\ifarXiv
\section{Details in Proof of Approx-zCDP Filter}

\subsection{Equivalence of Approximate zCDP Definitions}
\label{sec:equiv}
We will show that our definition of approximate zCDP is equivalent to the original definition of approximate zCDP due to \citet{bun2016concentrated}. Let us first restate their definition as a condition on a private algorithm $A$.

\begin{condition}[Original definition of \citet{bun2016concentrated}]\label{def:original}
For any neighboring datasets $x, x'$, there exist events $E$ and $E'$ such that {for all} $\lambda\geq 1$,
\begin{align*}
&D_\lambda(A(x) \mid E \| A(x') \mid E') \leq \rho\lambda, \\
&D_\lambda(A(x') \mid E' \| A(x) \mid E) \leq \rho \lambda,\\
&\P(A(x) \in E) \geq 1 - \delta, \text{ and}\\ &\P(A(x') \in E') \geq 1 - \delta.
\end{align*}
\end{condition}

Our definition is adapted from the approximate R\'enyi differential privacy definition due to \citet{Papernot022}. We restate the (unconditional) definition below.

\begin{condition}[Adapted from \citet{Papernot022}]\label{def:convex}
For any neighboring datasets $x, x'$, there exist distributions $P', P'', Q', Q''$ such that the outputs are distributed according to the following mixture distributions:
\[
 A(x)  \sim (1 - \delta)P' + \delta P'' ,\qquad A(x') \sim (1 - \delta)Q' + \delta Q''
\]
with $\text{for all }\lambda\geq 1,~ D_\lambda(P' \| Q') \leq \rho \lambda$ and $D_\lambda(P' \| Q') \leq \rho \lambda$.
\end{condition}

% \rrcomment{Can we change "Condition" to "Definition"?}
\begin{thm}
\label{thm:equiv}
Conditions \ref{def:original} and \ref{def:convex} are equivalent.
\end{thm}
\begin{proof}[\textbf{Proof of Theorem~\ref{thm:equiv}}]
Fix any neighbors $x, x'$. Suppose an algorithm $A$ satisfies Condition \ref{def:original} for some events $E, E'$. Then we could let $P'$ and $Q'$ be the conditional distributions $\P(A(x) \in \cdot \mid A(x) \in E)$ and $\P(A(x') \in \cdot \mid  A(x') \in E')$ respectively. Then let 
\begin{align*}
P''(\cdot) &= \frac{1}{\delta}\Big( \P(A(x) \in \cdot \mid A(x) \in E^c) \P(A(x) \in E^c) \\
&+ P'(\cdot) \ (\P(A(x) \in E) - (1 - \delta))\Big), \\
Q''(\cdot)  &= \frac{1}{\delta}\Big( \P(A(x') \in \cdot \mid A(x') \in E'^c)\P(A(x') \in E'^c) \\
&+ Q'(\cdot) \ (\P(A(x')\in E') - (1 - \delta))\Big)    .
\end{align*}
Then $A(x)$ is distributed according to the mixture $(1- \delta)P' + \delta P''$, and $A(x')$ is distributed according to the mixture $(1- \delta)Q' + \delta Q''$. Thus, $A$ also satisfies condition \ref{def:convex} given that $D_\lambda(P' \| Q') \leq \lambda \rho$ and 
$D_\lambda(Q' \| P') \leq \lambda \rho$ by our assumption of Condition \ref{def:original}.

Now suppose $A$ satisfies Condition \ref{def:convex} for some pairs of distributions $(P', P'')$ and $(Q', Q'')$. Then we can view the output distribution of $A(x)$ as generating a Bernoulli random variable $C$ such that with probability $(1 - \delta)$, $C=1$ and $A(x)$ draws an outcome from $P'$ and with probability $C=0$ and $A(x)$ draws an outcome from $P''$. Similarly, we can view $A(x')$ as flipping a coin $C'$ such that $A(x')$ draws an outcome from $Q'$ when $C'=1$. Then letting the events $E$ be all the randomness of $A(x)$ such that $C=1$ and $E'$ be all the randomness of $A(x')$ such that $C'=1$ satisfies condition \ref{def:original}.
\end{proof}

% \subsection{Approximate RDP Filter}

% \swcomment{add discussion that we extend to approximate RDP as well}

\subsection{Missing Proofs}
The following proof technique was used in prior works, including \cite{cesar2021bounding, feldman2020individual}
\begin{lem}\label{claim:nsm}
Let $(M_n^{(\lambda)})_{n \geq 1}$ be as defied in Equation \eqref{process2}. Then, $(M_n^{(\lambda)})_{n \geq 1}$ is a non-negative supermartingale with respect to its natural filtration $(\calF'_n)_{n \geq 1}$ given by $\calF_n := \sigma(Y_m' : m \leq n)$.
\end{lem}

\begin{proof}
 For any $k \geq 1$,
\begin{align*}
 \E[M_{n}^{(\lambda)} \mid \calF'_{n-1}] & = \E \Bigg[M_{n-1}^{(\lambda)} \, \exp\Bigg((\lambda - 1) \log\left(\frac{P'_{n}(Y_n' \mid Y_{1:n-1}' )}{Q'_{n}(Y_{n}' \mid Y_{1:n-1}' )}\right) - \lambda (\lambda - 1) \rho_{n}(Y_{1:n-1}')\Bigg)\mid \calF'_{n-1}\Bigg]\\
&= M_{n-1}^{(\lambda)}\, \E\left[\left(\frac{P'_{n}(Y_n' \mid Y_{1:n-1}' )}{Q'_{n}(Y_n' \mid Y_{1:n-1}' )}\right)^{(\lambda - 1)} \mid \calF'_{n-1}\right] \cdot \exp(-\lambda (\lambda - 1) \rho_{n}(Y'_{1:n-1}))\\
 &\leq M_{n-1}^{(\lambda)} \,  \exp(\lambda (\lambda - 1) \rho_{n}(Y_{1:n-1}')) \, \exp(-\lambda (\lambda - 1) \rho_{n}(Y_{1:n-1}')) \\
 &= M_{n-1}^{(\lambda)},
\end{align*}
where the last inequality follows from the R\'[enyi divergence bound due to approximate zCDP.
\end{proof}

\begin{lem}\label{lem:convex-combo-proof}
Let the distributions $P_{1:n}, Q_{1:n}, P'_{1:n}, Q'_{1:n}$ be defined in \eqref{eq:P}, \eqref{eq:Q}  for any $n\geq 1$. Then there exists distributions $P''_{1:n}$ and $Q''_{1:n}$ such that
\begin{align*}
    P_{1:n} = (1 - \delta) P'_{1:n} + \delta P''_{1:n},\\
        Q_{1:n} = (1 - \delta) Q'_{1:n} + \delta Q''_{1:n}.\\
\end{align*}
\end{lem}

\begin{proof}
We will show the decomposition for $P_{1:n}$, and the proof follows identically for the decomposition of $Q_{1:n}$. First, we can express $P_{1:n}(y_1, \cdots, y_n)$ for any $y_1, \cdots y_n$ as follows:
\begin{align*}
 P_{1:n}(y_1, \cdots, y_n) &= \prod_{m=1}^{n} P_m(y_m \mid y_{1:m-1}) \\
 & = \prod_{m = 1}^{n} \big[(1 - \delta_m(y_{1:m-1})) P'_{m}(y_m \mid y_{1:m-1})+ \delta_m(y_{1:m-1}) P''_{m}( y_m \mid y_{1:m-1}) \big]\\
 &= \sum_{S\subseteq [n] } \underbrace{\left(\prod_{m\in S} \delta_m(y_{1:m-1}) \prod_{m \in S^c}(1 - \delta_m(y_{1:m-1}))\right)}_{w_S(y_{1:m})} \cdot \underbrace{\prod_{m\in S}P''_m( y_m \mid y_{1:m-1})  \prod_{m\leq n, m\notin S} P'_m(y_m \mid y_{m-1})}_{f_S(y_{1:m})}
 \end{align*}
It suffices to show that $w_\emptyset(y_{1:m}) \geq 1 - \delta$ for all $y_{1:m}$.   To see this, we have the following by assumption 
$$w_\emptyset = \prod_{m\leq n} (1 - \delta_m(y_{1:m-1})) \geq 1 - \sum_{m\leq n} \delta_m(y_{1:m-1}) \geq 1 - \delta.$$
\end{proof}
\fi
\section{An Alternative Proof for Theorem \ref{thm:fully_adaptive}}
\label{sec:alterproof}

\iffalse We now provide a privacy filter that matches the rate of advanced composition. Our filter improves on the rate of the original filter presented in \citet{rogers2016odometer}. In addition, we improve over the filters of \citet{feldman2020individual} in that we assume the algorithms being composed satisfy \textit{conditional differential privacy}, whereas they assume \textit{conditional probabilistic differential privacy} when approximating advanced composition, otherwise needing to pay the potentially hefty conversion price outlined in Lemma~\ref{lem:dp_to_pdp}.\footnote{In Section 4.3 of their work, \citet{feldman2020individual} apply their R\'enyi filters to algorithms which satisfy conditional pDP. In general, a conversion from $(\epsilon, \delta)$-DP to $(\epsilon, \delta)$-pDP may be required to apply their filter.} \fi

\begin{comment}
However, we can prove the validity of our filter without under-the-hood conversions between differential privacy and R\'enyi differential privacy \cite{mironov2017renyi}, offering a more straightforward proof using the time-uniform bounds presented in \cite{howard2020line}.
\end{comment}
We begin by providing an alternative statement to Theorem~\ref{thm:fully_adaptive}, which is fully stated in terms of $\epsilon$'s and $\delta$'s. Straightforward calculations can confirm the equivalence of the two statements.
\begin{thm}
\label{thm:fully_adaptive_alt}
Suppose $(A_n)_{n \geq 1}$ is a sequence of algorithms such that, for any $n \geq 1$, $A_n$ is $(\epsilon_n, \delta_n)$-differentially private conditioned on $A_{1:n-1}$. Let $\epsilon > 0$  and $\delta = \delta' + \delta''$ be target privacy parameters such that $\delta' > 0, \delta'' \geq 0$. Consider the function $N : \R_{\geq 0}^\infty \times \R_{\geq 0}^\infty \rightarrow \N$ given by

$$
N((\epsilon_n)_{n \geq 1}, (\delta_n)_{n \geq 1}) := \inf\left\{n  : \epsilon < \sqrt{2\log\left(\frac{1}{\delta'}\right)\sum_{m \leq n + 1}\epsilon_m^2} + \frac{1}{2}\sum_{m \leq n + 1}\epsilon_m^2 \quad\text{ or }\quad \delta'' <  \sum_{m \leq n + 1}\delta_m\right\}.
$$
Then, the algorithm $A_{1:N(\cdot)}(\cdot) : \calX \rightarrow \calY^\infty$ is $(\epsilon, \delta)$-DP, where $N(x) := N((\epsilon_n(x))_{n \geq 1}, (\delta_n(x))_{n \geq 1})$. In other words, $N$ is an $(\epsilon, \delta)$-privacy filter.
\end{thm}
\begin{comment}
In Appendix~\ref{app:zcdp}, we extend Theorem~\ref{thm:fully_adaptive} to the setting where some of the algorithms satisfy \textit{conditional} zCDP. Observe first that if all privacy parameters are fixed in advance, i.e.\! $\epsilon_n$ and $\delta_n$ are constants, then, for any $\delta' > 0$, taking $\epsilon := \sqrt{2\log\left(\frac{1}{\delta'}\right)\sum_{m \leq n}\epsilon_m^2} + \frac{1}{2}\sum_{m \leq n}\epsilon_m^2$ and $\delta'' := \sum_{m \leq n}\delta_m$ recovers advanced composition up to low order terms (See Figure~\ref{fig:gap:gap}). Second, determining whether or not to stop at time $n$ only depends on the privacy parameters known at time $n$. Lastly, being fully aware of the stopping rule, a user can easily choose $\epsilon_N$ and $\delta_N$ to exactly meet the condition defining $N$ with equality, ensuring that the privacy budget of $(\epsilon, \delta)$ is fully utilized.

As a first step in our alternative proof of Theorem~\ref{thm:fully_adaptive}, it is easier to consider the case where each algorithm $A_n$ satisfies conditional $(\epsilon_n, \delta_n)$-pDP, as this condition provides a high-probability bound on the privacy loss. This allows us to use the martingale machinery in Appendix~\ref{app:concentration} to prove tight composition results.
\end{comment}
We first prove Theorem~\ref{thm:fully_adaptive_alt} under a stronger assumption on the algorithms being composed.
\begin{lem}
\label{lem:pdp_fil}
Theorem~\ref{thm:fully_adaptive_alt} holds under the stronger assumption that, for any $n \geq 1$, $A_n$ is $(\epsilon_n, \delta_n)$-pDP conditioned on $A_{1:n - 1}$.
\end{lem}
\begin{comment}
\iffalse \begin{proof}[\textit{sketch}]
Assume $\delta_n = 0$ for all $n$. For neighbors $x, x' \in \calX$, let $(M_n)_{n \in \N}$ be the privacy loss martingale, as constructed in Equation~\eqref{priv_martin}. For $n \geq 1$, letting $\Delta M_n := M_n - M_{n - 1}$, it is clear $\Delta M_n$ is $\epsilon_n^2$-subGaussian given $A_{1:n - 1}(x)$. As such, for any $a, b > 0$, Theorem~\ref{fact:line_cross} yields
\(
\P\left(\exists n \in \N : M_n > \frac{b}{2} + \frac{b}{2a}\sum_{m \leq n}\epsilon_m^2\right) \leq \exp\left(-\frac{b^2}{2a}\right).
\)
Carefully optimizing $a$ and $b$, bounding privacy loss, and decomposing $M_n$ yields that, with probability $\geq 1 - \delta'$, for all $n \leq N$,
\begin{align*}
\calL_{1:n}(x, x') \leq  \sqrt{2\sum_{m \leq n}\epsilon_m^2\log\left(\frac{1}{\delta'}\right)} + \frac{1}{2}\sum_{m \leq n}\epsilon_m^2 \leq \epsilon.
\end{align*}
Lastly, a careful union bound argument generalizes the above to the case of general $\delta_n$.
\end{proof}\fi

% \section{Bounding Privacy Loss in Adaptive Composition}
% \label{app:proof}
\end{comment}
To prove Lemma~\ref{lem:pdp_fil}, we need to following bound on the conditional expectation of privacy loss, which can be immediately obtained from the bound on expected privacy loss presented in \citet{bun2016concentrated}.

\begin{lem}[Proposition 3.3 in \citet{bun2016concentrated}]
\label{lem:priv_loss_bound}
Suppose $A$ and $B$ are algorithms such that $A$ is $\epsilon$-differentially private conditioned on $B$. Then, for any input dataset $x \in \calX$ and neighboring dataset $x' \sim x$, we have that
$$
\E\left(\calL(x, x') | B(x)\right) \leq \frac{1}{2}\left(\epsilon(B(x))\right)^2.
$$
\end{lem}

Now, we prove Lemma~\ref{lem:pdp_fil}.

\begin{proof}[\textbf{Proof of Lemma~\ref{lem:pdp_fil}}]
To begin, we assume that the algorithms $(A_n)_{n \geq 1}$ satisfy $(\epsilon_n, 0)$-pDP conditioned on $A_{1:n - 1}$. We will show how to alleviate this assumption on the approximation parameter in the second half of the proof. Fix an input database $x \in \calX$. For convenience, we denote by $(\calF_n(x))_{n \in \N}$ the natural filtration generated by $(A_n(x))_{n \geq 1}$. Since we have fixed $x \in \calX$, for notational simplicity, we write $\epsilon_n$ for the random variable $\epsilon_n(A_{1:n -1}(x))$ and define $\delta_n$ similarly. Additionally, by $N$ we mean the stopping time $N((\epsilon_n)_{n \in \N}, (\delta_n)_{n \in \N})$. Recall that we have already argued that, for any neighboring dataset $x' \sim x$, the process
{\small
\begin{align*}
M_n := M_n(x, x') = \calL_{1:n}(x, x') - \sum_{m \leq n}\E\left(\calL_m(x, x') | \calF_{m - 1}(x)\right)
\end{align*}
}%
is a martingale with respect to $(\calF_n(x))_{n \in \N}$. Further observe that its increments $\Delta M_n := \calL_n(x, x') - \E\left(\calL_n(x, x')|\calF_{n - 1}(x)\right)$ are $\epsilon_n^2$-subGaussian conditioned on $\calF_{n - 1}(x)$.

Thus, by Theorem~\ref{fact:line_cross}, we know that, for any $b, a > 0$, we have
$$
\P\left(\exists n \in \N : M_n \geq \frac{b}{2} + \frac{b}{2a}V_n\right) \leq \exp\left(\frac{-b^2}{2a}\right),
$$
where the process $(V_n)_{n \in \N}$ given by $V_n := \sum_{m \leq n}\epsilon_m^2$ is the accumulated variance up to and including time $n$. Thus, it suffices to optimize the free parameters $a$ and $b$ to prove the result.

To do this, consider the following function $f : \R_{\geq 0} \rightarrow \R_{\geq 0}$ given by
$$
f(y) = \sqrt{2\log\left(\frac{1}{\delta'}\right)y} + \frac{1}{2}y.
$$
Clearly, $f$ is a quadratic polynomial in $\sqrt{y}$ which is strictly increasing. In particular, one can readily check that 
\begin{equation}
\label{eq:opt}
y^\ast :=  \left(-\sqrt{2\log\left(\frac{1}{\delta'}\right)} + \sqrt{2\log\left(\frac{1}{\delta'}\right) + \epsilon}\right)^2
\end{equation}
solves the equation $f(y) = \epsilon$, where $\epsilon > 0$ is the target privacy parameter. 

As such, setting $a := y^\ast$ and $b := \sqrt{2\log\left(\frac{1}{\delta'}\right)y^\ast}$ yields
$$
\exp\left(\frac{-b^2}{a}\right) = \exp\left(\frac{-2y^\ast\log\left(\frac{1}{\delta'}\right)}{y^\ast}\right) = \delta'.
$$

Furthermore, expanding the definition of $(M_n)_{n \in \N}$, we see that for the selected parameters the parameters yield, with probability at least $1 - \delta'$, for all $n \leq N$ we have:
{\small
\begin{align*}
&\calL_{1:n}(x, x') \leq \frac{b}{2} + \frac{b}{2a}V_n + \sum_{m \leq n}\E\left(\calL_m(x, x') \mid \calF_{m - 1}\right) \\
&\leq \frac{b}{2} + \frac{b}{2a}\sum_{m \leq n}\epsilon_m^2 + \frac{1}{2}\sum_{m \leq n}\epsilon_m^2\\
&= \frac{1}{2}\sqrt{2\log\left(\frac{1}{\delta'}\right)y^\ast} + \frac{1}{2}\frac{\sqrt{2\log\left(\frac{1}{\delta'}\right)y^\ast}}{y^\ast}\sum_{m \leq n}\epsilon_m^2 + \frac{1}{2}\sum_{m \leq n}\epsilon_m^2 \\
&\leq \frac{1}{2}\sqrt{2\log\left(\frac{1}{\delta'}\right)y^\ast} + \frac{1}{2}\sqrt{2\log\left(\frac{1}{\delta'}\right)y^\ast} + \frac{1}{2}\sum_{m \leq n}\epsilon_m^2 \\
&= \sqrt{2\log\left(\frac{1}{\delta'}\right)y^\ast} + \frac{1}{2}\sum_{m \leq n}\epsilon_m^2 \leq \sqrt{2\log\left(\frac{1}{\delta'}\right)y^\ast} + \frac{1}{2}y^\ast = \epsilon.
\end{align*}
}%
Thus, we have proven the desired result in the case where all algorithms have $\delta_n = 0$.

Now, we show how to generalize our result to the case where the approximation parameters $\delta_n$ are not identically zero. Define the events
\begin{align*}
&A := \left\{\exists n \leq N : \calL_{1:n}(x, x') > \epsilon\right\}, \text{ and} \\
&B := \left\{\exists n \leq N: \calL_n(x, x') > \epsilon_n\right\}.
\end{align*}
Our goal is to show that, with $N$ defined as in the statement of Theorem~\ref{thm:fully_adaptive}, that $\P(A) \leq \delta$. Simply using Bayes rule, we have that
$$
\P(A) = \P(A \cap B^c) + \P(A \cap B) \leq \P(A|B^c) + \P(B) \leq \delta' + \P(B),
$$
where the second inequality follows from our already-completed analysis in the case that $\delta_n = 0$. Now, we show that $\P(B) \leq \delta''$, which suffices to prove the result as we have, by assumption, $\delta = \delta' + \delta''$. 

Define the modified privacy loss random variables $(\widetilde{\calL}_n(x, x'))_{n \in \N}$ by
$$
\widetilde{\calL}_n(x, x') := \begin{cases} \calL_n(x, x') \qquad n \leq N \\
0 \qquad \text{otherwise}
\end{cases}.
$$
Likewise, define the modified privacy parameter random variables $\widetilde{\epsilon}_n$ and $\widetilde{\delta}_n$ in an identical manner. Then, we can bound $\P(B)$ in the following manner:
\begin{align*}
    &\P(\exists n \leq N : \calL_n(x, x') > \epsilon_n) = \P\left(\exists n \in \N : \widetilde{\calL}_n(x, x') > \widetilde{\epsilon}_n\right) \\
    &\leq \sum_{n = 1}^\infty \P\left(\widetilde{\calL}_n(x, x') > \widetilde{\epsilon}_n \right) = \sum_{n = 1}^\infty \E\P\left(\widetilde{\calL}_n(x, x') > \widetilde{\epsilon}_n | \mathcal{F}_{n- 1}\right)\\
    &\leq \sum_{n = 1}^\infty \E\widetilde{\delta}_n = \E\left[\sum_{n = 1}^\infty \widetilde{\delta}_n\right] =\mathbb{E}\left[\sum_{n \leq N}\delta_n\right] \leq \delta''.
\end{align*}
Thus, we have have proven the desired result in the general case.\end{proof}

Our key insight above is to view filters as functions of the ``intrinsic time" determined by privacy parameters, $\sum_{m \leq n}\epsilon_m^2$. Lemma~\ref{lem:pdp_fil} can also be obtained leveraging the analysis for R\'enyi filters \citep{feldman2020individual}. However, our approach to proving Lemma~\ref{lem:pdp_fil} has the advantage that it does not require reductions between different modes of privacy. While Lemma~\ref{lem:priv_loss_bound}, which bounds expected privacy loss, does require some complicated analysis, we only ever need to apply Lemma~\ref{lem:pdp_fil} to instances of randomized response, in which case computing the privacy loss bound is trivial. 

We now use Lemma~\ref{lem:pdp_fil} to prove Theorem~\ref{thm:fully_adaptive_alt}. Recall that Lemma~\ref{lem:dp_to_pdp} shows that algorithms that satisfy pDP also satisfy DP, but the converse is not true and may require a conversion cost. To avoid this cost, we define following generalization of randomized response.
%For randomized response, the space of datasets is simply $\{0, 1\}$. With some small probability, the algorithm reveals the hidden bit, revealing $\top$ for $1$ or $\bot$ for $0$. Otherwise, the algorithm flips a slightly biased coin, keeping the hidden bit the same with some probability, and flipping the bit with some comparable but smaller probability. 

\begin{definition}[Conditional Randomized Response]
\label{def:rr}
Let $\calR := \{0, 1, \top, \bot\}$ and $2^\calR$ be the corresponding power set of $\calR$. Then, $R$ taking inputs in $\{0, 1\}$ to outputs in the measurable space $(\calR, 2^\calR)$ is an instance of $(\epsilon, \delta)$-randomized response if, for $b \in \{0, 1\}$, $R(b)$ outputs the following:
$$
R(b) = \begin{cases} b  &\text{with probability } (1 - \delta)\frac{e^\epsilon}{1 + e^\epsilon} \\
1 - b &\text{with probability } (1 - \delta)\frac{1}{1 + e^\epsilon}\\
\top &\text{with probability } \delta \text{ if } b = 1\\
\bot &\text{with probability } \delta \text{ if } b = 0.
\end{cases}
$$
More generally, suppose $B : \{0, 1\} \rightarrow \calZ$ is a randomized algorithm. For functions $\epsilon, \delta : \calZ \rightarrow \R_{\geq 0}$, we say $R$ is an instance of $(\epsilon, \delta)$-randomized response conditioned on $B$ if, for any true input $b' \in \{0, 1\}$ and hypothesized alternative $b \in \{0, 1\}$, the conditional probability $\P(R(b) \in \cdot | B(b') = z)$ is the same as the law of $(\epsilon(z), \delta(z))$-randomized response with input bit $b$.
\end{definition}

Conditional $(\epsilon, \delta)$-randomized response satisfies both conditional $(\epsilon, \delta)$-DP and conditional $(\epsilon, \delta)$-pDP. We will leverage the fact that it satisfies both privacy definitions with the same parameters. A surprising result in the nonadaptive setting is that \textit{any} $(\epsilon, \delta)$-DP algorithm can be viewed as a randomized post-processing of $(\epsilon, \delta)$-randomized response \citep{kairouz2015composition}. We generalize this result to the adaptive conditional setting below. In the language of Blackwell's comparison of experiments \citep{blackwell1953equivalent}, instances of randomized response are ``sufficient" for instances of arbitrary DP algorithms, and we prove that the same is true for conditional randomized response and conditionally DP algorithms. In what follows, by a transition kernel $\nu$, we mean that for any  $b \in \calZ$ and $r \in \calR$, $\nu(\cdot, r \mid b)$ is a probability measure on $(\calY, \calG)$.

\begin{lem}[Reduction to Conditional Randomized Response]
\label{lem:cond_post_rr}
Let $A$ and $B$ map from $\calX$ to measurable spaces $(\calY, \calG)$ and $(\calZ, \calH)$, respectively. Suppose $A$ is $(\epsilon, \delta)$-differentially private conditioned on $B$. Fix neighbors $x_0, x_1 \in \calX$, and let $R$ be an instance of $(\epsilon, \delta)$-randomized response conditioned on $B'$, where $B' : \{0, 1\} \rightarrow \calZ$ is the restricted algorithm satisfying $B'(b) = B(x_b)$. Then, there is a transition kernel $\nu : \calG \times \calR \times \calZ \rightarrow [0, 1]$ such that, for all $b, b' \in \{0, 1\}$, $\P\left(A(x_b) \in \cdot \mid B'(b')\right) = \nu_{b, b'}$, where $\nu_{b, b'} = \E\left(\nu(\cdot, R(b) \mid B'(b')) \mid B'(b')\right)$.\footnote{\label{foot:avg_cond}By $\nu_{b, b'}(\cdot) := \E\left(\nu(\cdot, R(b) \mid B'(b')) \mid B'(b')\right)$, we mean that $\nu_{b, b'}$ is the (random) averaged probability measure: \begin{align*}
\nu_{b, b'}(\cdot) &= \P(R(b) = 1 \mid B'(b'))\nu(\cdot, 1 \mid B'(b')) \\
&+ \P(R(b) = 0 \mid B'(b'))\nu(\cdot, 0 \mid B'(b')) \\
&+ \P(R(b) = \bot \mid B'(b') )\nu(\cdot, \bot \mid B'(b')) \\
&+ \P(R(b) = \top \mid B'(b'))\nu(\cdot, \top \mid B'(b')).
\end{align*}
}
\end{lem}

%Lemma~\ref{fact:post_proc_rr} is a very powerful --- in fact it can be used to prove advanced composition \citep{kairouz2015composition}. However, we need a more detailed result to hold. This is because we want to allow privacy parameters to be adaptively selected. In Appendix~\ref{app:rand_resp}, we describe a generalization of randomized response to a conditional setting, wherein the parameters $\epsilon$ and $\delta$ are selected according to auxiliary information. We then show, in analogue to Lemma~\ref{fact:post_proc_rr}, that \textit{any} conditionally differentially private algorithm reduces to an instance of conditional randomized response.
%\vspace*{-0.5\baselineskip}
Lemma~\ref{lem:cond_post_rr} tells us that the conditional distribution obtained by averaging the kernel $\nu(\cdot , R(b) \mid B'(b'))$ over the randomness in $R(b)$ matches the conditional distribution of $A(x_b)$.
To prove Lemma~\ref{lem:cond_post_rr}, first recall the important fact that \textit{any} differentially private algorithm can be viewed as a post-processing of randomized response \citep{kairouz2015composition}, as stated in Lemma~\ref{fact:post_proc_rr} below.

\begin{lem}[Reduction to Randomized Response \citep{kairouz2015composition}]
\label{fact:post_proc_rr}
Let algorithm $A : \calX \rightarrow \calY$ be $(\epsilon, \delta)$-DP. Let $R$ be an instance of $(\epsilon, \delta)$-randomized response. Then, for any neighbors $x_0, x_1 \in \calX$, there is a transition kernel $\nu : \calG \times \calR \rightarrow [0,1]$ such that for $b \in \{0, 1\}$, we have
$\P(A(x_b) \in \cdot ) = \nu_b$, where\footnote{\label{foot:avg} By $\nu_b(\cdot) := \E \nu(\cdot, R(b))$, we mean $\nu_b$ is the averaged probability measure given by 
\begin{align*}
\nu_b(\cdot) &= \P(R(b) = 1)\nu(\cdot, 1) + \P(R(b) = 0)\nu(\cdot, 0) \\
&+ \P(R(b) = \bot)\nu(\cdot, \bot) + \P(R(b) = \top)\nu(\cdot, \top).
\end{align*}
}
$\nu_b  = \E \nu(\cdot, R(b))$.

\end{lem}

In Lemma~\ref{lem:cond_post_rr} of Section~\ref{sec:filter}, we generalized Lemma~\ref{fact:post_proc_rr} to the case of conditional differential privacy. To do this, we introduced \textit{conditional randomized response} in Definition~\ref{def:rr}. In conditional randomized response, on the event $\{B = z\}$, the conditional laws of $R(0)$ and $R(1)$ just become that of regular randomized response with some known privacy parameters $\epsilon(z)$ and $\delta(z)$. We now prove Lemma~\ref{lem:cond_post_rr}.\\

\begin{proof}[\textbf{Proof of Lemma~\ref{lem:cond_post_rr}}]
Let $b, b' \in \{0, 1\}$ be arbitrary. For any outcome $\{B'(b') = z\}$, let $\P_z(A(x_b) \in \cdot)$ be the probability measure $\P(A(x_b) \in \cdot | B'(b') = z)$. In particular, this measure does not depend on the input bit $b'$. By the assumptions of conditional differential privacy (Definition~\ref{def:cdp}), it follows that under the probability measure $\P_z$, $A(x_b)$ is $(\epsilon(z), \delta(z))$-differentially private.  Moreover, it also follows that $R$ is an instance of $(\epsilon(z), \delta(z))$-randomized response under $\P_z$. Consequently, Lemma~\ref{fact:post_proc_rr} yields the existence of a kernel $\nu_z$ such that $\P_z(A(x_b) \in \cdot) = \E_z \nu_z(\cdot, R(b))$, where the averaged measure is as defined in Footnote~\ref{foot:avg}. Setting $\nu(\cdot, R(b) | z) := \nu_{z}(\cdot, R(b))$, we see that
$$
\P(A(x_b) \in \cdot \mid B'(b') = z) = \E\left(\nu(\cdot, R(b) \mid z) \mid B'(b') = z\right),
$$
which thus yields
$$
\P(A(x_b) \in \cdot \mid B'(b')) = \E\left(\nu(\cdot, R(b) \mid B'(b')) \mid B'(b')\right),
$$
where the conditionally averaged measure is as described in Footnote~\ref{foot:avg_cond} in the main body of the paper. This proves the desired result.
\end{proof}

Lastly, before proving Theorem~\ref{thm:fully_adaptive_alt}, we need the following lemma. This lemma essentially tells us that if $A$ is $(\epsilon, \delta)$-pDP conditioned on $B$, and $A'$ is a randomized post-processing algorithm, then releasing the vector $(A, A')$ is also $(\epsilon, \delta)$-pDP conditioned on $B$. Note that this is \textit{not} in contradiction with the converse direction of Lemma~\ref{lem:dp_to_pdp}, as releasing the output of $A'$ alone may not satisfy conditional $(\epsilon, \delta)$-pDP. But once we observe $A$, since $A'$ is a post-processing, we can gleam no more information about the true underlying dataset.

\begin{lem}
\label{lem:post_proc_pdp}
Suppose $A, B$ are algorithms with inputs in $\calX$ and outputs in measurable spaces $(\calY, \calG)$ and $(\calZ, \calH)$ respectively. Assume $A$ is $(\epsilon, \delta)$-pDP conditioned on $B$. Let $(S, \calS) $ be a measurable space and suppose $\mu :\calS \times \calY \times \calZ \rightarrow [0, 1]$ is a conditional transition kernel. Suppose $A' : \calX \rightarrow S$ is an algorithm satisfying 
\begin{equation}
\label{eq:pp}
\P\left(A'(x) \in \cdot | A(x') = y, B(x') = z\right) = \mu(\cdot, y \mid z),
\end{equation}
for all $y \in \calY, z \in \calZ$, and $x, x' \in \calX$. Then, the joint algorithm $(A, A') : \calX \rightarrow \calY \times S$ is also $(\epsilon, \delta)$-pDP conditioned on $B$.
\end{lem}
\begin{proof}[\textbf{Proof of Lemma~\ref{lem:post_proc_pdp}}]
Let $x, x' \in \calX$ be arbitrary neighboring datasets. Let $q^x_B, q^{x'}_B$ be the corresponding conditional joint densities of $(A(x), A'(x))$ and $(A(x'), A'(x'))$ given $B(x)$ respectively. Likewise, let $p^x_B, p^{x'}_B$ be the corresponding conditional densities of $A(x)$ and $A(x')$ respectively conditioned on $B(x)$, and $q^x_{B, A}, q^{x'}_{B, A}$ the conditional densities of $A'(x)$ and $A'(x')$ given $A(x)$ and $B(x)$. Let $\calL_B^{(A, A')}(x, x')$ denote the joint privacy loss between $(A(x), A'(x))$ and $(A(x'), A'(x'))$ given $B(x)$, while $\calL_B^{A}(x, x')$ denotes the privacy loss between $A(x)$ and $A(x')$ given $B(x)$. We have, using Bayes rule,
\begin{align*}
    &\calL_{B}^{(A, A')}(x, x') = \log\left(\frac{q^x_B(A(x), A'(x)\mid B(x))}{q^{x'}_B(A(x), A'(x)\mid B(x))}\right) \\
    &= \log\left(\frac{p^x_B(A(x) \mid B(x))}{p^{x'}_B(A(x) \mid B(x))}\cdot\frac{q^x_{B, A}(A'(x)\mid B(x), A(x))}{q^{x'}_{B, A}(A'(x) \mid B(x), A(x))}\right) \\
    &= \log\left(\frac{p^x_B(A(x)\mid B(x))}{p^{x'}_B(A(x) \mid B(x))}\right) = \calL_B^{(A)}(x, x'),
\end{align*}
The first equality on the second line follows from the assumption outlined in Equation~\eqref{eq:pp}. More specifically, since we have
\begin{align*}
&\P\left(A'(x) \in \cdot | A(x),B(x)\right) = \mu(\cdot, A(x) \mid B(x)) = \\
&\P\left(A'(x') \in \cdot | A(x),B(x)\right),
\end{align*}
it follows that the conditional densities $q^x_{B, A}$ and $q^{x'}_{B, A}$ are equal almost surely. Since $A$ is $(\epsilon, \delta)$-pDP conditioned on $B$, the result now follows.
\end{proof}

We now can prove Theorem \ref{thm:fully_adaptive_alt} using these tools.

\begin{proof}[\textbf{Proof of Theorem~\ref{thm:fully_adaptive_alt}}]
Fix arbitrary neighbors $x_0, x_1 \in \calX$. Let $(R_n)_{n \geq 1}$ be a sequence of algorithms such that $R_n$ is an instance of $(\epsilon_n, \delta_n)$-randomized response conditioned on $A'_{1:n - 1}: \{0, 1\} \rightarrow \calY^{n - 1}$, where $A'_m : \{0, 1\} \rightarrow \calY$ is the restricted algorithm given by $A'_m(b) := A_m(x_b)$, for all $m \geq 1$. Lemma~\ref{lem:cond_post_rr} guarantees the existence of a sequence of transition kernels  $(\nu_n)_{n \geq 1}$, $\nu_n : \calG \times \calR \times \calY^{n - 1} \rightarrow [0, 1]$ such that, for all $n \geq 1$ and $b, b' \in \{0, 1\}$,
we have $\P(A_n'(b) \in \cdot \mid A'_{1:n - 1}(b')) = \nu_{b, b'}^{(n)}$ almost surely. Here, $\nu_{b, b'}^{(n)}$ is the averaged conditional probability, as defined in terms of $\nu_n$ in Lemma~\ref{lem:cond_post_rr} and Footnote~\ref{foot:avg_cond}. This equality means we can find an underlying probability space (i.e.\! a coupling) such that the random post-processing draws from the kernel $\nu_n(\cdot, R_n(b) \mid A'_{1:n -1}(b'))$ equal $A'_n(b)$ almost surely, for all $n \geq 1$.
%if $X_n(b) \sim \nu^{(n)}_{b, b'}$, we have $\P(A_n'(b) \in \cdot \mid A'_{1:n - 1}(b')) = \P(X_n(b) \in \cdot \mid A'_{1:n - 1}(b'))$ almost surely. In particular, we have that for any $n \geq 1$, $A'_{1:n}(b) =_d X_{1:n}(b)$, where $=_d$ denotes equality in distribution. This means we can find an underlying probability space (i.e.\! a coupling) such that, for all $n \geq 1$, $A'_n(b) = X_n(b)$ almost surely.

Now, for any $n \geq 1$, since $R_n$ is an instance of $(\epsilon_n, \delta_n)$-randomized response conditioned on $A'_{1:n - 1}$, it follows that $R_n$ is in fact $(\epsilon_n, \delta_n)$-pDP conditioned on $A'_{1:n - 1}$. Moreover, this also implies that $R_n$ is $(\epsilon_n, \delta_n)$-pDP conditioned on $(A'_{1:n - 1}, R_{1:n - 1})$, since, by definition, $\epsilon_n$ and $\delta_n$ only depend on the realizations of $R_{1:n - 1}$ through the outputs of $A'_{1:n -1}$. By Lemma~\ref{lem:post_proc_pdp}, it follows that for all $n \geq 1$, the algorithm $(R_n, A'_n)$ is $(\epsilon_n, \delta_n)$-pDP conditioned on $(R_{1:n - 1}, A'_{1:n-1})$. Thus, by Lemma~\ref{lem:pdp_fil}, it follows that the composed algorithm $(R_{1:N'(\cdot)}(\cdot), A'_{1:N'(\cdot)}(\cdot))$ is $(\epsilon, \delta)$-DP, where $N'(b) := N(x_b)$ and $\epsilon, \delta$ and $N$, are as outlined in the statement of Theorem~\ref{thm:fully_adaptive}. 

Lastly, since differential privacy is closed under arbitrary post-processing \citep{dwork2014algorithmic}, it follows that $A'_{1:N'(\cdot)}(\cdot)$ is $(\epsilon, \delta)$-differentially private. Since $x_0$ and $x_1$ were arbitrary neighboring inputs, the result follows, i.e. $A_{1:N(\cdot)}(\cdot) :\calX \rightarrow \calY^\infty$ is $(\epsilon, \delta)$-differentially private.
\end{proof}

\section{Proof for Privacy Odometers in Theorem~\ref{thm:new_odometers} \label{app:proof}}

We now show the formal proof for our privacy odometers presented in Theorem~\ref{thm:new_odometers} in Section~\ref{sec:odometers}.

\begin{proof}[\textbf{Theorem~\ref{thm:new_odometers}}]
As in the proof of Theorem~\ref{lem:pdp_fil}, we first consider the case where $\delta_n = 0$ for all $n \geq 1$. In this case, fix an input dataset $x \in \calX$ and a neighboring dataset $x' \in \calX$. Let $(M_n)_{n \in \N}$ be the corresponding privacy loss martingale as outlined in Equation~\eqref{priv_martin}, where we implicitly hide the dependence on $x, x'$, which are fixed. Let $(u_n)_{n \geq 1}$ be one of the sequences outlined in the theorem statement, and define $U_n := u_n(\epsilon_{1:n}, \delta_{1:n})$ for all $n \geq 1$, where once again we write $\epsilon_n$ and $\delta_n$ for $\epsilon_n(A_{1:n - 1}(x))$ and $\delta_n(A_{1: n -1}(x))$ respectively. It follows from Theorems~\ref{fact:line_cross}, \ref{fact:mixture}, and \ref{fact:stitch} that
$$
\P\left(\exists n \in \N : M_n > B_n\right) \leq \delta,
$$
for $B_n = U_n - \frac{1}{2}\sum_{m \leq n}\epsilon_m^2$. Recalling that $M_n = \sum_{m \leq n}\{\calL_m(x, x') - \E(\calL_m(x, x')|\calF_{n - 1}(x))\}$ and  that $\E(\calL_n(x, x')|\calF_{n - 1}(x)) \leq \frac{1}{2}\epsilon_n^2$ for all $n \in \N$, it thus follows that
$$
\P\left(\exists n \in \N : \calL_{1:n}(x, x') > U_n \right) \leq \delta,
$$
where $(\calF_n(x))_{n \geq 1}$ is again the natural filtration generated by $(A_n(x))_{n \geq 1}$. Thus, since $x \sim x'$ were arbitrary, we have shown that $(u_n)_{n \geq 1}$ is a $\delta$-privacy odometer in the case $\delta_n = 0$ for all $n \geq 1$.

To generalize to the case where $\delta_n$ may be nonzero, we can apply precisely the same argument used in the second part of the proof of Lemma~\ref{lem:pdp_fil}, thus proving the general result.
\end{proof}

\section{An Algorithm Satisfying $(\epsilon, \delta)$-DP but not $(\epsilon, \delta)$-pDP}
\label{app:fail_dp}

In this appendix, we construct a simple algorithm taking binary inputs that satisfies $(\epsilon, \delta)$-DP but not $(\epsilon, \delta)$-pDP. In particular, this provides intuition as to why we conjecture our odometers constructed in Section~\ref{sec:odometers} would not hold under the assumption that the algorithms being composed satisfy $(\epsilon, \delta)$-DP in general. 

To this end, fix a privacy parameter $\epsilon > 0$ and an approximation parameter $\delta \in (0, 1)$. Let $A : \{0, 1\} \rightarrow \{0, 1, \top, \bot\}$ be an instance of $(\epsilon, \delta)$-randomized response, and let $B : \{0,1\} \rightarrow \{0, 1\}$ be defined by
$$
B(b) := \begin{cases}1 \qquad \text{if } A(b) \in \{1, \top\}, \\
0 \qquad \text{otherwise. }
\end{cases}
$$
Since differential privacy is closed under arbitrary post-processing, it follows that the constructed algorithm $B$ is $(\epsilon, \delta)$-differentially private. On the other hand, setting $x = 1$, $x' = 0$, we note that on the event $\left\{B(1) = 1\right\}$,
\begin{align*}
\calL_B(1, 0) &= \log\left(\frac{\P(B(1) = 1)}{\P(B(0) = 1)}\right) = \log\left(\frac{\P(A(1) = 1) + \P(A(1) = \top)}{\P(A(0) = 1) + \P(A(0) = \top)}\right)\\
&=\log\left(\frac{\delta + (1 - \delta)\frac{e^\epsilon}{1 + e^\epsilon}}{(1 - \delta)\frac{1}{1 + e^\epsilon}}\right)\\
% &=\log\left(\frac{\delta}{1 - \delta} + e^\epsilon\left(1 + \frac{\delta}{1 - \delta}\right)\right) 
&= \log\left( \frac{\delta+e^\epsilon}{1-\delta} \right)
> \epsilon.
\end{align*}
Since straightforward calculation yields 
$$
\P(B(1) = 1) = (1 - \delta)\frac{e^{\epsilon}}{1 + e^\epsilon} + \delta > \delta,
$$
we see that $B$ does not satisfy $(\epsilon, \delta)$-pDP.

\end{document}